\documentclass{article}

\usepackage[main, final]{neurips_2026}

\usepackage[utf8]{inputenc} 
\usepackage[T1]{fontenc}    
\usepackage{hyperref}       
\usepackage{url}            
\usepackage{booktabs}       
\usepackage{amsfonts}       
\usepackage{nicefrac}       
\usepackage{microtype}      
\usepackage{xcolor}         

\usepackage{amsthm}
\usepackage{url}
\usepackage{titletoc}
\usepackage{chngcntr}
\usepackage{enumitem}
\usepackage{algorithm}
\usepackage[noend]{algpseudocode} 
\usepackage[makeroom]{cancel}

\usepackage{amsmath}
\usepackage{amssymb}
\usepackage{mathtools}
\usepackage[capitalize]{cleveref}

\theoremstyle{plain}
\newtheorem{theorem}{Theorem}[section]
\newtheorem{proposition}[theorem]{Proposition}

\theoremstyle{definition}

\theoremstyle{remark}
\newtheorem{remark}{Remark}

\usepackage[textsize=tiny]{todonotes}

\usepackage{amsmath,amsfonts,bm}
\usepackage{multirow,caption}

\newcommand{\R}{\mathbb R}

\newcommand{\E}{\mathbb E}

\newcommand{\tr}{\operatorname{tr}}

\DeclareMathOperator*{\argmin}{arg\,min}

\newcommand{\inv}{^{-1}}
\DeclareMathOperator*{\Var}{Var}

\newcommand{\loss}{\mathcal L}
\newcommand{\batch}{\mathcal B}
\newcommand{\D}{\mathrm D}
\newcommand{\prox}{\operatorname{prox}}
\newcommand{\fsp}{\mathcal F}
\newcommand{\Vsp}{\mathbb V}
\newcommand{\dv}{\mathrm d}
\newcommand{\rank}{\operatorname{rank}}
\newcommand{\fst}{f_{\operatorname{struct}}}

\title{Closed-Form Last Layer Optimization}

\author{%
  Alexandre Galashov\thanks{Equal contribution.}  \\
  Google DeepMind \\
  Gatsby Unit, University College London \\
  \texttt{agalashov@google.com} \\
  \And
  Nathaël Da Costa\footnotemark[1] \\
  Tübingen AI Center, University of Tübingen \\
  \texttt{nathael.da-costa@uni-tuebingen.de} \\
  \And
  Liyuan Xu \\
  Secondmind\thanks{Work done while at University College London} \\
  \texttt{liyuan9988@gmail.com} \\
  \And
  Philipp Hennig \\
  Tübingen AI Center, University of Tübingen \\
  \texttt{philipp.hennig@uni-tuebingen.de} \\
  \And
  Arthur Gretton  \\
  Google DeepMind \\
  Gatsby Unit, University College London \\
  \texttt{gretton@google.com} \\
}

\begin{document}

\maketitle

\begin{abstract}
    Neural networks are typically optimized with variants of stochastic gradient descent. Under a squared loss, however, the optimal solution to the linear last layer weights is known in closed-form. We propose to leverage this during optimization, treating the last layer as a function of the backbone parameters, and optimizing solely for these parameters. We show this is equivalent to alternating between gradient descent steps on the backbone and closed-form updates on the last layer. We adapt the method for the setting of stochastic gradient descent, by trading off the loss on the current batch against the accumulated information from previous batches. We provide theoretical analyses showing convergence of the method to an optimal solution in the neural tangent kernel regime, as well as quantifying the gains compared to standard SGD in a one-step analysis. Finally, we demonstrate the effectiveness of our approach compared with SGD and Adam on a squared loss in several regression tasks, including neural operators and causal inference.
\end{abstract}

\begin{figure}[h]
    \begin{minipage}{0.30\textwidth}
        \includegraphics[width=1\textwidth]{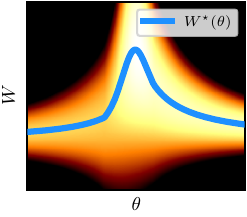}
    \end{minipage}
    \hfill
    \begin{minipage}{0.67\textwidth}
    \captionsetup{singlelinecheck=false}
    \caption[Loss landscape of a two-parameter neural network]{%
      The squared loss landscape of a two-parameter neural network $$f(x) = W\operatorname{ReLU}(\theta x)$$
    with three random training data points. Dark / light regions correspond to values of high / low loss respectively. We plot in blue the optimal last layer parameter $W^\star(\theta)$ as a function of the backbone parameter $\theta$. We propose to optimize along the blue curve, rather than in two-dimensional space.}\label{fig:2d-nn}
    \end{minipage}
\end{figure}

\section{Introduction}
Training deep neural networks is almost always done with variants of stochastic gradient descent (SGD). Despite their empirical success, these iterative methods treat every layer of the network in the same way. However, the linear last layer often admits a much simpler -- and in the case of squared loss, closed-form -- optimal solution. This mismatch suggests an opportunity: if the optimal last layer weights can be computed directly given the current features produced by the backbone, we can regard the last layer as an implicit function of the backbone parameters. This could simplify the optimization problem by constraining the last layer to be optimal throughout (see \cref{fig:2d-nn}).

In SGD, gradients at each step are computed in minibatches. Because of computational constraints, the closed-form solution of the last layer should also use minibatches. This risks overfitting the last layer to each batch at each optimization step. To correct for this issue, there is the need to account for previous last layer solutions.

In this paper, we develop a training procedure that can perform optimization with a closed-form optimal last layer through SGD on the backbone parameters. Our contributions are as follows:
\begin{enumerate}[wide, labelwidth=!, labelindent=0pt]
    \item We propose leveraging the closed-form last layer solution for squared loss, and optimizing the backbone parameters while treating the last layer as a deterministic function of those parameters. We show that this requires no backpropagation through the closed-form solution (\cref{sec:method}).
    \item We adapt the approach to stochastic mini-batch training by regularizing for previous last layer solutions, producing a practical algorithm that integrates cleanly with standard training pipelines and that admits an approximate Kalman filter interpretation (\cref{sec:stochastic}).
    \item We provide theoretical analyses. In the infinite width neural tangent kernel (NTK) limit we show that, under continuous time dynamics, despite non-convexity of the loss in function space, convergence of the method to an optimal solution is guaranteed. We further quantify the superiority of our method compared to SGD in the discrete time case with a one-step analysis (\cref{sec:theory}).
    \item We validate the approach empirically, demonstrating improvements compared to standard training under squared losses, including applications in deep feature instrumental variable regression and Fourier neural operators (\cref{sec:experiments}).
\end{enumerate}

\subsection{Related Work}

We outline several strands of related work.

\textbf{Two-timescale regime.} Optimizing under a closed-form last layer can be seen as performing  bilevel optimization, where an optimization problem is nested into another \citep{zhang_introduction_2024,petrulionyte_functional_2024}. In recent years, this last layer bilevel optimization approach has been considered in several works as a simplifying assumption for demonstrating convergence of gradient descent in neural networks. This framework was coined the two-timescale regime \citep{marion_leveraging_2023,berthier_learning_2024,bietti_learning_2025,barboni_ultra-fast_2025}.

\cite{marion_leveraging_2023} noted that, by the envelope theorem, optimizing with an optimal last layer as a function of the backbone parameters is equivalent to optimizing only the backbone parameters while keeping the last layer optimal. In the present work, we bring this theoretical argument to a practical method that can accelerate optimization, by observing that the envelope result can be leveraged computationally in backpropagation. Indeed, unlike these works, our aim is to propose novel methodology, and demonstrate its practicality in a number of scenarios. From the theoretical side, our work is the first to analyze the critical points of the resulting loss in function space, and the convergence to a global minimum in the NTK regime. Other works have investigated the mean-field regime instead \citep{wang_mean-field_2024,takakura_mean-field_2024}.

In an experiment, \cite{barboni_ultra-fast_2025} propose to update the last layer by an exponential moving average of closed-form solutions to account for the stochasticity in SGD. However, this approach decouples the last layer from the backbone, as they no longer optimize the same loss, which leads to instabilities. In contrast, our approach for stochasticity allows the last layer and the backbone to continue optimizing for the same loss.

\textbf{Layer-wise learning.} \cite{singh_layer-specific_2015,you_large_2017} propose to tune the learning rates of SGD layer-wise. \cite{you_large_2017} further show that, while shallower layers tend to have smaller gradients, this is not a reason for such layers needing larger learning rates, as the layer weights also appear to be smaller themselves. Indeed, \cite{chen_which_2022} analyze the convergence speeds of different layers during optimization, and show that, despite smaller gradients, shallower layers learn faster than deeper layers. Our method thus remediates this layer convergence bias for the last layer.

\textbf{Bayesian last layers.} These works leverage closed-form or quasi-closed-form solutions for last layer Bayesian posteriors to train them with variational inference \citep{harrison_variational_2023,brunzema_bayesian_2024,harrison_heteroscedastic_2025}. In contrast, the present work does not construct a Bayesian posterior, and instead leverages closed-form solutions for the last layer to accelerate optimization.

\textbf{Feature learning in instrumental variables regression.} Training with a closed-form last layer has found applications in causal inference, particularly instrumental variables regression \citep{xu_learning_2020,xu_deep_2021}. In these applications, the closed-form last layer solution is required to solve the bilevel optimization problem. However, these works backpropagated through the closed-form solution, making them computationally expensive. Moreover, they did not include a proximal regularizer, and thus required large batch sizes for the networks to train. We address both limitations in the present work, and demonstrate superior performance of our method in these applications. \cite{chen2025unified} addresses the same limitations by removing the closed-form stage altogether, using a first-order Lagrangian reformulation of the bilevel problem with global convergence guarantees for two-layer networks. We instead keep the closed-form last layer but avoid differentiating through it.

\textbf{Randomized neural networks.} Randomized neural networks fix randomly generated hidden weights and solve for the output weights in closed form. Examples include the random vector functional link (RVFL) network~\citep{pao1992functional,pao1994learning}, its deep and ensemble deep extensions dRVFL and edRVFL~\citep{shi2021random}, the broad learning system (BLS)~\citep{chen2018broad}, which expands the network in width, and stochastic configuration networks (SCNs)~\citep{wang2017stochastic}, which sample hidden weights under a supervisory mechanism. We refer to~\cite{malik2023random} for a survey of recent variants of these models. We view these as a degenerate case of our method, where hidden weights are random and only the last layer is solved in closed form. In contrast, our work is optimization-centric: the last layer and hidden weights are updated in a coupled way, and closed-form solves on mini-batches let the method scale to large datasets. A systematic comparison with state-of-the-art randomized networks is left for future work.

\textbf{Neural tangent kernel (NTK) regime.} NTK analysis typically assumes convexity of the loss in function space to show convergence to an optimum \cite{jacot_neural_2018}. The present work goes beyond the standard analysis to show convergence to an optimum for our loss, which is shown to be non-convex in function space.

\section{Background}

\paragraph{Non-linear multi-dimensional regression.}
We consider the regression problem with a squared loss in which we aim to predict $y \in \R^o$ from the input $x\in \mathcal X$ (typically $\mathcal{X} \subset \R^m$), where $\mathcal X$ is some input space. We employ a
model $f(x; W, \theta) = W\phi_\theta(x)$, where $\phi_\theta\colon \mathcal X \to \R^d$ is the neural network feature map, which we call the \emph{backbone}, parametrised by $\theta\in \Theta$, $\Theta \subset \R^N$ is the feature parameter space, and $W \in \mathbb{R}^{o\times d}$ is the weight matrix for the last linear layer. Given $n$ observations $\{(x_i, y_i)\}_{i=1}^n$, we want to find the parameters $(W,\theta)$ that minimize the regularized squared loss
\begin{equation}\label{eq:old-loss}
    \loss(W,\theta) = \sum_{i=1}^n\|y_i - W\phi_\theta(x_i)\|_2^2 + \beta\|W\|^2_F
\end{equation}
for some hyperparameter $\beta> 0$, where $\|\cdot\|_F$ is the Frobenius norm.
For a fixed $\theta$, the objective~\eqref{eq:old-loss} is a ridge regression problem for $W$ which enjoys the closed-form solution (derivation in \cref{sec_app:ridge_derivation})
\begin{equation}\label{eq:closed-form-w}
    \begin{aligned}
            W^\star(\theta)&:= \argmin_{W\in \R^{o\times d}}\mathcal L(W,\theta)
            =
            Y\phi_\theta(X)^\top\left(\phi_\theta(X)\phi_\theta(X)^\top+\beta I\right)\inv
    \end{aligned}
\end{equation}
where $X = (x_1,\dots,x_n)$, $Y = (y_1,\dots,y_n) \in \R^{o\times n}$ and $\phi_\theta(X)\in \R^{d\times n}$ is a matrix where we apply $\phi_{\theta}$ to each $x_i$.

\paragraph{Gradient descent.} An approach to optimize~\eqref{eq:old-loss} for $W$ and $\theta$ is to use gradient descent, an iterative optimization algorithm which starts from $({W}_0, {\theta}_0)$ and at each iteration $t$ solves the following linearized optimization problem
\begin{equation}
    \label{eq:linearized_problem}
    \begin{aligned}
            {W}_{t+1},{\theta}_{t+1} = \argmin_{W,\theta} \nabla\loss({W}_{t},{\theta}_{t})^\top [W,\theta]
            + \frac{1}{2\eta} ||W-{W}_{t}||_{F}^2 + \frac{1}{2\alpha} ||\theta - {\theta}_{t}||_{F}^2,
    \end{aligned}
\end{equation}
where $(\alpha, \eta)$ are learning rates and $[W,\theta]$ denotes the concatenation operation. The solution to~\eqref{eq:linearized_problem} can be expressed with the familiar gradient descent updates
\begin{equation}\label{eq:gd-update}
    \begin{aligned}
        {W}_{t+1} = {W}_{t} - \eta\nabla_{W}\mathcal{L}({W}_{t}, {\theta}_t), \qquad
        {\theta}_{t+1} = {\theta}_{t} - \alpha \nabla_{\theta}\mathcal{L}({W}_{t}, {\theta}_t). 
    \end{aligned}
\end{equation}

\section{Optimizing with a Closed-Form Last Layer}\label{sec:method}
During optimization, we would like to leverage the fact that, for each $\theta$, the optimal last layer $W^\star(\theta)$ is available in closed-form (\cref{eq:closed-form-w}). The idea is that there is no need to update $ W_t$ through gradient steps as in \cref{eq:gd-update}, as we may treat it directly as a function of $\theta$ through \cref{eq:closed-form-w}. This leads to the loss
\begin{equation}\label{eq:new-loss}
    \begin{aligned}
        \loss^\star(\theta) &:= \loss(W^\star(\theta),\theta)
        =
        \sum_{i=1}^n\|y_i - W^\star(\theta)\phi_\theta(x_i)\|_2^2 + \beta\|W^\star(\theta)\|^2_F
    \end{aligned}
\end{equation}
We now propose to optimize this loss instead of \cref{eq:old-loss}. Computationally, this involves alternating between solving linear regressions to obtain $W^\star(\theta)$ and gradient steps on $\theta$ through $\mathcal L$ (backpropagating through the closed-form solution to the regression). Explicitly, we start at some $\theta_0$ and iterate
\begin{equation}\label{eq:new-gd-update}
    \textstyle
    {W}_{t+1} = W^\star(\theta_t), \qquad {\theta}_{t+1} = {\theta}_{t} - \alpha \nabla_{\theta}\loss^\star({\theta}_t).
\end{equation}
Note that $\nabla_{\theta}\loss^\star({\theta}_t)$ involves backpropagating through $W^\star(\theta)$ given by \cref{eq:closed-form-w},
and hence through an inverse. This is computationally demanding. Fortunately, this operation is not needed, as the following result demonstrates (see also the \emph{envelope theorem} and \citep[Remark 1]{marion_leveraging_2023}):
\begin{proposition}\label{thm:gradient-equiv}
    For fixed $\theta$, letting $W^\star := W^\star(\theta)$ with \cref{eq:closed-form-w}, we have
    \begin{equation}
        \nabla_\theta \loss^\star(\theta) = \nabla_\theta \loss(W^\star, \theta)
    \end{equation}
\end{proposition}
\begin{proof}
    By the chain rule,
    \begin{equation}
        \begin{aligned}
            &\nabla_\theta \loss^\star(\theta)
            =
            \underbrace{\nabla_W\loss(W^\star,\theta)}_{=0}\D W^\star(\theta)^\top+\nabla_\theta\mathcal L(W^\star,\theta)\underbrace{\D\operatorname{id}(\theta)^\top}_{=I}
            =
            \nabla_\theta\mathcal L(W^\star,\theta)
        \end{aligned}
    \end{equation}
    where $\nabla_W\loss(W^\star,\theta) = 0$ since
    $W^\star = \argmin_W \loss(W,\theta)$ and $\D$ denotes the differential operator.
\end{proof}
Unlike $\nabla_\theta \loss^\star(\theta)$ which requires complex backpropagation, $\nabla_\theta \loss(W^\star, \theta)$ uses standard backpropagation through $\phi_\theta$ with the fixed last layer $W^\star$. From \Cref{thm:gradient-equiv}, \cref{eq:new-gd-update} is equivalent to
\begin{equation}\label{eq:final-gd-update}
    {W}_{t+1} = W^\star(\theta_{t}), \qquad{\theta}_{t+1} = {\theta}_{t} - \alpha \nabla_{\theta}\loss( W_{t+1},\theta_t),
\end{equation}
i.e.~it suffices to replace the gradient step on $W$ in \cref{eq:gd-update} by a closed-form update of the form \cref{eq:closed-form-w}.

\section{The Practical Method}\label{sec:stochastic}
In practice, neural networks are not trained with gradient descent, but with \emph{stochastic} gradient descent, or variants thereof. At time $t$ we observe a batch of data $\batch_t \subset \{(x_i, y_i)\}_{i=1}^n$. Then the squared loss on the batch is given by
\begin{equation}
    \label{eq:l2_loss_batch}
    \loss_{\batch_t} (W,\theta) := \sum_{(x_i,y_i)\in \batch_t} \|y_i - W\phi_\theta(x_i)\|_2^2 +\beta\|W\|_F^2.
\end{equation}
As before, we may write
\begin{equation}
    \label{eq:stochastic_closed_form}
    W_{\batch_t}^\star(\theta) := \argmin_{W\in \R^{o\times d}}\loss_{\batch_t}(W,\theta)
\end{equation}
and update
\begin{equation}\label{eq:sgd-update}
    \textstyle
    {W}_{t+1} = W^\star_{\batch_t}(\theta_{t}), \qquad{\theta}_{t+1} = {\theta}_{t} - \alpha \nabla_{\theta}\loss_{\batch_t}( W_{t+1},\theta_t).
\end{equation}
While such an approach will be valid for large batch sizes, it will be ineffective for small batches as the last layer $W_{\batch_t}^\star(\theta)$ will overfit to each batch $\batch_t$ at each $t$. The last layer estimates $ W_{t+1}$ will then vary drastically at each iteration. As a consequence, the features $\phi_{\theta_t}(X)$ might not be able to adapt to an unstable last layer. The smaller the batch size relative to the complete dataset, the more severe the issue (see the experiments in \cref{sec:experiments}).


Instead, we propose to optimize a different loss. Motivated by how gradient descent regularizes to the previous estimates of the parameters (see \cref{eq:linearized_problem}) we propose to regularize the objective function against the distance from $W$ to the previous estimate $ W_t$, yielding the \emph{proximal loss}
\begin{equation}
    \label{eq:l2_prox}
    \loss^{\prox}_{\batch_t,  W_t} (W, \theta) := \sum_{(x_i,y_i)\in \batch_t} \|y_i - W\phi_\theta(x_i)\|_2^2 + \lambda\|W -  W_t\|_F^2
\end{equation}
where $\|\cdot\|_F$ is the Frobenius norm and $\lambda>0$ is some hyperparameter. This ensures that closed-form solutions to~\cref{eq:l2_prox} are close to the previous estimate $W_t$.

As before, we define (derivation in \cref{sec_app:proximal_derivation})
\begin{equation}\label{eq:closed-form-proximal}
\begin{aligned}
        W^\star_{\mathcal B_t, W_t}(\theta) := \argmin_{W\in \R^{o\times d}}\loss^{\prox}_{\batch_t, W_t} (W,\theta)
    = \big(Y_t\phi_\theta(X_t)^\top+  \lambda W_t\big)\left(\phi_\theta(X_t)\phi_\theta(X_t)^\top+\lambda I\right)\inv
\end{aligned}
\end{equation}
where $X_t$ and $Y_t$ correspond to the matrix of inputs and outputs of batch $\batch_t$ respectively, and
\begin{equation}
    \loss^{\prox\star}_{\batch_t, W_t}(\theta) := \loss^{\prox}_{\batch_t, W_t} (W^\star_{\batch_t, W_t}(\theta),\theta).
\end{equation}
Thus we propose to start at some $( W_0, \theta_0)$ and iterate
\begin{equation}\label{eq:new-sgd-update}
    \textstyle
    {W}_{t+1} = W^\star_{\mathcal B_t, W_t}(\theta_t), \qquad {\theta}_{t+1} = {\theta}_{t} - \alpha \nabla_{\theta}\loss^{\prox\star}_{\batch_t, W_t}({\theta}_t).
\end{equation}
This approach addresses the stochasticity issues and ensures that $W$ and $\theta$ optimize the same loss, namely $\loss^{\prox}$. Moreover, we can obtain a result analogous to \cref{thm:gradient-equiv} for the proximal loss:
\begin{proposition}\label{thm:stochastic-gradient-equiv}
    For fixed $\theta$, letting $W^\star_{\batch_t, W_t} := W^\star_{\batch_t, W_t}(\theta)$, we have
    \begin{equation}
        \nabla_\theta\loss^{\prox\star}_{\batch_t, W_t}(\theta) = \nabla_\theta\loss_{\batch_t}(W^\star_{\batch_t, W_t},\theta).
    \end{equation}
\end{proposition}
\begin{proof}
    As for \cref{thm:gradient-equiv}, we have $\nabla_\theta\loss^{\prox\star}_{\batch_t, W_t}(\theta) = \nabla_\theta \loss^{\prox}_{\batch_t, W_t} (W^\star_{\batch_t, W_t},\theta)$. Since the regularizer $\|W^\star_{\batch_t,W_t}-  W_t\|$ does not depend on $\theta$, we have $\nabla_\theta \loss^{\prox}_{\batch_t,  W_t} (W^\star_{\batch_t, W_t},\theta) = \nabla_\theta\loss_{\batch_t}(W^\star_{\batch_t, W_t},\theta)$.
\end{proof}
Like \cref{thm:gradient-equiv}, \cref{thm:stochastic-gradient-equiv} replaces the demanding backpropagation of $\nabla_\theta\loss^{\prox\star}_{\batch_t, W_t}(\theta)$ by a standard backpropagation step for $\nabla_\theta\loss_{\batch_t}(W^\star_{\batch_t},\theta)$. So \cref{eq:new-sgd-update} is equivalent to
\begin{equation}\label{eq:final-sgd-update}
    {W}_{t+1} = W^\star_{\mathcal B_t, W_t}(\theta_t), \qquad {\theta}_{t+1} = {\theta}_{t} - \alpha \nabla_{\theta}\loss_{\batch_t}(  W_{t+1},{\theta}_t).
\end{equation}
In \cref{sec_app:kalman} we give another interpretation for \cref{eq:final-sgd-update} as performing approximate Kalman filtering on the last layer throughout SGD on the backbone parameters.

\subsection{Numerical considerations}
\label{subsec:numerical_considerations}

We now discuss practical choices required to implement the updates~\eqref{eq:final-sgd-update} in a practical algorithm.


\paragraph{Use of a bias term.} We could add an additional bias dimension to feature vector $\phi_{\theta}$  and a learnable bias $b$ to the last layer $W$, which lead to extended vectors $\tilde{\phi}_{\theta} = [\phi_{\theta},1]$ and $\tilde{W} =[W,b]$. Overall, we found it was always beneficial for empirical performance to add bias.



\paragraph{Last layer initialization.}
First, we consider \emph{zeros}, i.e. $W_0^{ij} = 0,\forall i,j$, which worked the best in practice. We then consider classical weight initializations -- \emph{LeCun normal} $W_0^{ij}\sim \mathcal{N}\left(0, \frac{1}{d} \right)$, \emph{Xavier normal} $W_0^{ij} \sim \mathcal{N}\left(0, \frac{2}{d + o} \right)$ and \emph{He normal} $W_0^{ij} \sim \mathcal{N}\left(0, \frac{2}{d} \right)$. Bias is always initialized as $b_0^i =0$. See Appendix~\ref{sec:initialization_impact} for detailed ablations.

\paragraph{Full algorithm.} The complete method is given in Algorithm~\ref{alg:method}. The loop is structured so the last layer is re-fit after each gradient update, ensuring that the last layer is up-to-date at the end of each iteration.

\begin{algorithm}[H]
\caption{Proximal closed-form SGD}
\begin{algorithmic}[1] 
    \State \textbf{Given:} Batch size $B$, proximal coefficient $\lambda > 0$, neural network $\phi_{\theta}$ with initial parameters $\theta_0$, learning rate $\alpha > 0$, initial last layer parameters $W_0$.
    \State $t \gets 0$
    \State \textbf{Fit last layer on the first batch}
    \State ${W}_{1} \leftarrow W^\star_{\mathcal B_1,W_0}(\theta_t)$
    \While{${\theta}_t$ has not converged}
    \State $t \gets t+1$
    \State \textbf{Update backbone on the current batch $\batch_t$}
    \State $\theta_{t} \gets {\theta}_{t-1} - \alpha \nabla_{\theta}\loss_{\batch_t}(  W_{t},{\theta}_{t-1})$
    \State \textbf{Update last layer on the next batch $\batch_{t+1}$}
    \State ${W}_{t+1} \leftarrow W^\star_{\mathcal B_{t+1},{W}_{t}}(\theta_t)$
\EndWhile
    \State \textbf{Output:} Optimized $(W^\star, \theta^\star)$
\end{algorithmic}
\label{alg:method}
\end{algorithm}

We present an alternative method in \cref{sec_app:alternative_algo} as \cref{alg:method_simple}. In this alternative method, the gradient and closed-form solutions are computed on the same batch within each iteration, making it simpler to implement. However it departs from our theoretical framework laid out in \cref{eq:final-sgd-update} and \cref{thm:stochastic-gradient-equiv}.

\subsection{Application to Classification}

While the above methodology is mainly introduced for regression problems with $\ell_2$ loss, here we discuss how it could be extended to classification settings. Please see Appendix~\ref{sec_app:classification_results} for the experimental results on classification.

We treat the output $y_i$ as one-hot vectors, i.e., $y_i \in \{0, 1\}^C$ such that $\sum_{c=1}^{C} [y_i]_c = 1$, where $C=o$ is the number of classes. We then can directly optimize the a squared loss on one-hot vectors (see for instance \cite{hui2021evaluation} for the use of the squared loss in classification) using~\eqref{eq:final-sgd-update} to optimize $W$ and $\theta$. This, however does not guarantee that the model $f(x; W, \theta) = W\phi_\theta(x)$ outputs probability vectors, i.e.~$\sum_{c=1}^{C} W_{c\cdot} \phi_\theta(x) \neq 1$, where $W_{c\cdot}$ is the $c$\textsuperscript{th} row of $W$. Therefore, for prediction, we take the arg max over output vectors, i.e.~$c_{\text{out}}(x) = \arg\max_c W_{c\cdot}\phi_{\theta}(x)$. While this strategy is a simple heuristic, we found that using it together with~\cref{eq:final-sgd-update} led to reasonable performance.

\cite{hui2021evaluation} proposes changes in the parametrization of the squared loss that empirically improve the performance on classification; in \cref{sec_app:gen_derivation} we derive a formula for closed-form solutions for these squared loss generalizations. These can be expensive to compute for large number of classes, so we do not make use of them in this work.

\section{Theoretical Analyses}\label{sec:theory}

\subsection{Analysis of the Modified Loss}\label{sec:loss_theory}
In the section we uncover theoretical insights for the loss $\loss^\star(\theta)$ from \cref{eq:new-loss}. For a tractable analysis, we will start by considering this loss as a function of the backbone, instead of the parameters $\theta$.

Let $\fsp$ be the space of functions $\phi\colon \mathcal X \to \R^d$. We define the loss~\eqref{eq:old-loss} with a backbone $\phi$ as second argument:
\begin{equation}\label{eq:loss-fsp}
    \loss_\fsp(W,\phi) = \sum_{i=1}^n\|y_i-W\phi(x_i)\|_2^2 + \beta\|W\|_F^2
\end{equation}
where $W\in\R^{o\times d}$, $\phi \in \fsp$ and $\beta\geq 0$. Then, as before we define $W^\star_\fsp(\phi):= \argmin_{W\in\R^{o\times d}}\loss_\fsp(W,\phi)$ and $\loss^\star_\fsp(\phi) := \loss^\star_\fsp(W^\star_\fsp(\phi),\phi)$.

The following theorem shows that $\loss^\star_\fsp$ possesses unexpected characteristics.
\begin{theorem}\label{thm:non-convex}
    If $Y\neq 0$ then $\loss^\star_\fsp$ is not convex and admits critical points $\phi^\star$ that are not global minimizers.
\end{theorem}

The proof for this theorem can be found in \cref{sec_app:modified_loss}.

In contrast, the usual squared (or ridge) loss $\sum_{i=1}^n \|y_i -f(x_i)\|^2_2$, where $f\colon \mathcal{X} \rightarrow \mathbb{R}^{o}$ which does not use a closed-form solution on the last layer $W^\star_\fsp(\phi)$, is convex in $f$ . The critical points of this loss function are exactly the functions $f^\star$ such that $f^\star(x_i) = y_i$ for all $i$ (or $f^\star(x_i) \approx y_i$ in the presence of a ridge regularizer). The non-trivial critical points of $\loss^\star_\fsp$ occur because, when the features $\phi(X)_{j \cdot}$ are orthogonal to the outputs $Y_{k \cdot}$ for all $j$ and $k$, there is no gradient information for the features. For example $\phi^\star = 0$ is always a critical point of $\loss^\star_\fsp$.

Is this an issue when $\phi$ is a neural network? Intuitively, these `bad' critical points form a `lower dimensional manifold' in function space, so should be avoided with high probability. We formalize this intuition by analyzing the NTK infinite width neural network regime \cite{jacot_neural_2018}. In this regime, the initial neural network function $\phi$ can be shown to be a Gaussian process with respect to the random initialization, controlled by the neural Gaussian process kernel (NGPK). The training dynamics of $\phi$ are given by kernel gradient descent in function space $\fsp$ with respect to the NTK. If the NTK is positive definite, we know that $\phi$ will converge to a critical point of $\loss^\star_\fsp$ in $\fsp$. The following result shows that if we make the slightly stronger assumption that the NGPK is positive definite (see for example \citep[Theorem 4.5]{gao_wide_2023}), then $\phi$ will converge to a global minimizer.

\begin{theorem}\label{thm:ntk}
    In the NTK regime with positive definite NGPK, assuming $\beta=0$, $d \geq \rank Y$ and that the $x_i$ are distinct, we have that $\phi$ converges almost surely to a global minimizer $\phi^\star$ of $\loss^\star_\fsp$.
\end{theorem}

The assumptions $\beta=0$, $d\geq \rank Y$ and that the $x_i$ are distinct simplify the proof, the latter two ensuring that the outputs $Y$ are expressible through the features $\phi(X)$. This theorem provides a sanity check that the non-convexity of the function space loss (\cref{thm:non-convex}) causes no problems during optimization. See \cref{sec_app:ntk} for a formal description of the NTK regime and a proof.

\subsection{Analysis of One-Step Performance}\label{sec:performance_theory}
\Cref{sec:loss_theory} uncovers theoretical insights on the modified loss $\loss^\star$. It shows non-convexity in function space, but which does not appear to be an issue, using the NTK regime as a case study. However, it does not say anything about the gains of optimizing $\loss^\star$ as opposed to $\loss$.

To investigate this, we compare the decrease in loss $\loss(W,\theta)-\loss(W_{GD},\theta)$ after a gradient step on the last layer with the decrease in loss $\loss(W,\theta)-\loss(W_{CF},\theta)$ after a closed-form step. The following result quantifies the superiority of closed-form steps.

\begin{theorem}\label{thm:one-step}
    Fix $W$ and $\theta$, and assume $\beta= 0$ and that $W\not\in \argmin_{\tilde W} \loss(\tilde W,\theta)$. Then for any learning rate $\eta>0$,
    \begin{equation}
        \frac{\loss(W,\theta)-\loss(W_{GD},\theta)}{\loss(W,\theta)-\loss(W_{CF},\theta)} \leq \left(1+ \frac{\Var_{\mathbb P}{(\sigma^2)}}{\E_{\mathbb P}[\sigma^2]^2} \right)^{-1} \leq 1
    \end{equation}
    
    where $W_{GD} := W-\eta\nabla \loss(W,\theta)$, $W_{CF} := W^\star(\theta)$ and $\mathbb P$ is a categorical distribution over the singular values $\sigma$ of $\phi_\theta(X)$, independent of $\eta$. If we write $\phi_\theta(X) = U\Sigma V^\top$ for its SVD and if all singular values are distinct, $\mathbb P$ is given by
    \begin{equation}
        \mathbb P(\sigma = \Sigma_{ii}) = \frac{\|((W\phi_\theta(X)-Y)V)_{\cdot i}\|^2}{\|(W\phi_\theta(X)-Y)V\|_F^2}.
    \end{equation}
\end{theorem}
The proof for \cref{thm:one-step} can be found in \cref{sec_app:one-step}. Intuitively, the more spread out the singular values of $\phi_\theta(X)$, the worse conditioned is the quadratic problem for $W$, the larger is the coefficient of variation $\frac{\sqrt{\Var_{\mathbb P}{(\sigma^2)}}}{\E_{\mathbb P}[\sigma^2]}$, and the worse is a gradient step compared to a closed-form step. This bound is valid for optimal learning rates in GD, so in practice the discrepancy may be much greater.

\section{Experiments}\label{sec:experiments}
\label{sec:experiments}

\paragraph{Experimental setup.} We evaluate the performance of our proximal closed-form approach (\cref{eq:final-sgd-update}), denoted as ``\emph{$\ell_2$ c.f.~proximal ($\lambda)$}''. We compare this against a closed-form ridge regression baseline (\cref{eq:sgd-update}), ``\emph{$\ell_2$ c.f.~ridge ($\beta)$}'', and a standard ``\emph{$\ell_2$ loss}'' baseline optimized via \cref{eq:gd-update}.
We expect ``\emph{$\ell_2$ c.f.~proximal ($\lambda)$}'' to be effective across all batch sizes while ``\emph{$\ell_2$ c.f.~ridge ($\beta)$}'' to only work well with large batch sizes.
We sweep over all the hyperparameters such as $\lambda$ and $\beta$ as well as the learning rate $\alpha$ for every method. We train models on a training set and we select best hyper-parameters using a validation set. The performance metrics are reported on the hold-out \emph{test} set, averaged over $3$ random seeds (with $95\%$ confidence intervals). See \cref{sec_app:experimental_details} for experimental details. We refer the reader to \cref{sec_app:additional_results} for additional experiments and ablations, and \cref{sec_app:wall_clock} for timings.

\paragraph{Quantum Chemistry regression.} The QM9 dataset~\citep{ramakrishnan2014quantum} consists of approximately 133,000 small organic molecules, each represented by 435-dimensional Coulomb Matrix features. The task is to predict 12 quantum chemical properties. We train all methods for 50 epochs with various batch sizes and report test-set mean squared error (MSE), using either SGD or Adam to optimize the backbone. See Appendix~\ref{sec_app:experimental_details} for further details.

\begin{figure}[h]
\begin{center}
\includegraphics[width=5.3in]{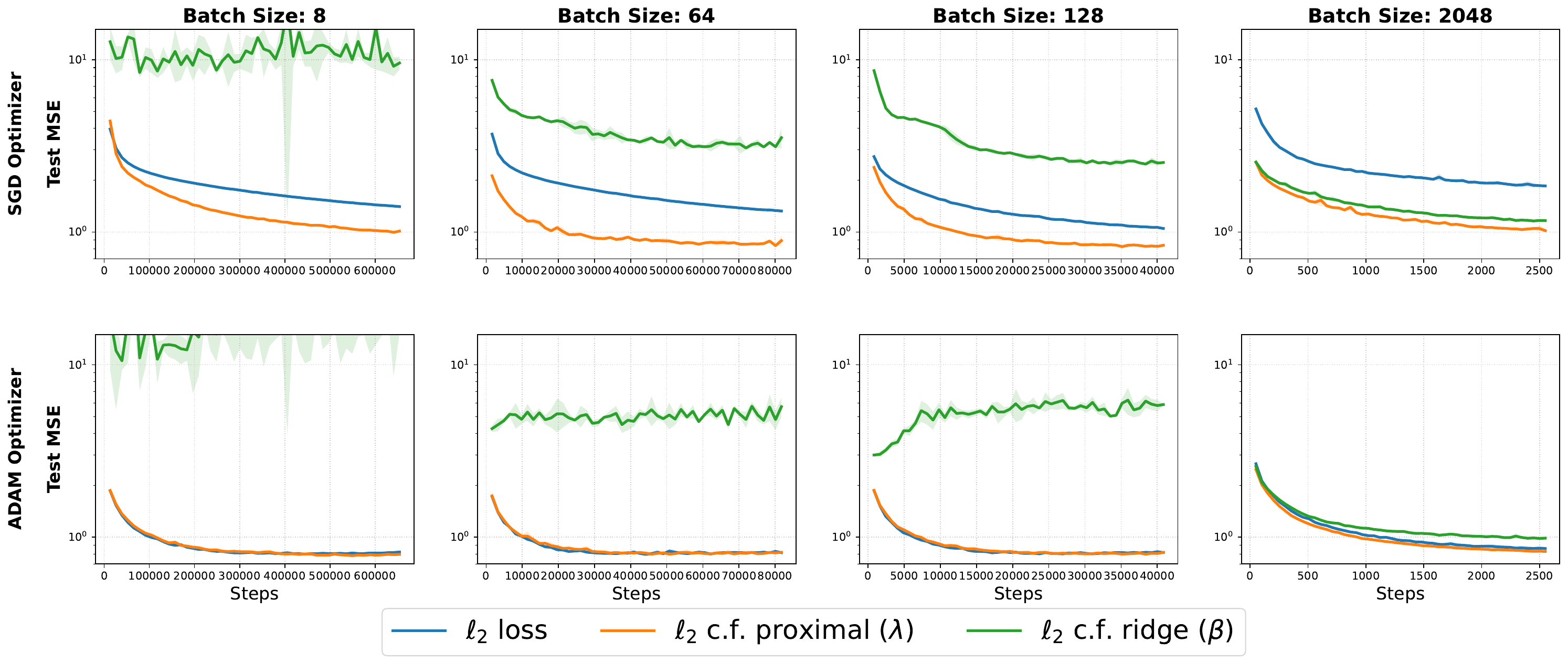}
\end{center}
\caption{\textbf{QM9 results}. X-axis is the number of iterations, Y-axis is a test set mean squared error (MSE), columns represents different batch sizes, rows represent different backbone optimizers. Different colors indicate different methods.}
\label{fig:qm9_results}
\end{figure}

Results are illustrated in \Cref{fig:qm9_results}. We observe that ``\emph{$\ell_2$ c.f.~ridge ($\beta)$}'' exhibits significant instability and fails to converge to a reasonable MSE at small batch sizes. Only at a batch size of $2048$ does it attain performance comparable to ``\emph{$\ell_2$ c.f.~proximal ($\lambda)$}''. This is expected: a large batch already approximates the full dataset well, narrowing the gap between the ridge and proximal formulations. In contrast, our proximal method ``\emph{$\ell_2$ c.f.~proximal ($\lambda)$}'', consistently outperforms the standard ``\emph{$\ell_2$ loss}'' when paired with SGD and performs comparably under Adam, demonstrating robustness to the choice of base optimizer.

\paragraph{Impact of $\lambda$.} We study sensitivity of ``\emph{$\ell_2$ c.f.~proximal ($\lambda)$}'' to the $\lambda$ parameter on QM9. Smaller batch sizes require larger values of $\lambda$ for stable training: intuitively, a stronger proximal penalty prevents the closed-form solution from overfitting to noisy small-batch statistics. For larger batch sizes, where the batch better approximates the full dataset, smaller values of $\lambda$ suffice. Overall, we find that moderately large $\lambda$ (around $10^3$) is beneficial across settings.

\begin{figure}[h]
\begin{center}
\includegraphics[width=3.3in]{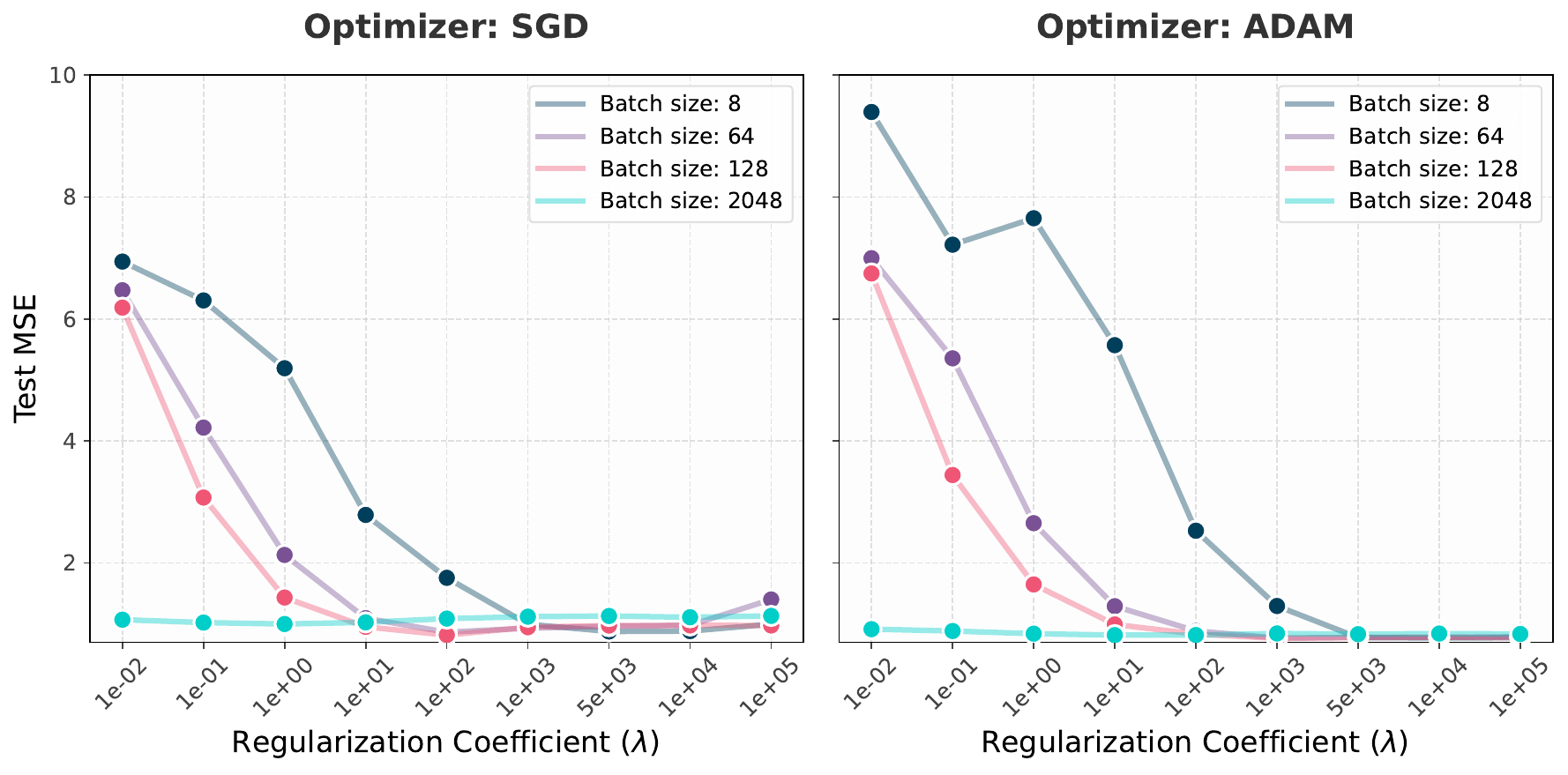}
\end{center}
\caption{\textbf{QM9, impact of $\lambda$}. X-axis is $\lambda$, Y-axis is the final test mean squared error. First column shows results with the SGD optimizer, while second column shows results with the Adam optimizer.}
\label{fig:qm9_results}
\end{figure}

\paragraph{Fourier Neural Operator (FNO) regression.}

We apply our approach to operator learning using a 1D Fourier Neural Operator (FNO)~\citep{fno} on the viscous Burgers equation
\begin{equation}
    \frac{\partial u}{\partial t} + u\frac{\partial u}{\partial x} = \nu \frac{\partial^2 u}{\partial x^2}
\end{equation}
on the periodic domain $\Omega = (0, 2 \pi)$ and viscosity $\nu=0.1$. Our dataset consists of $2048$ initial conditions $u(t=0, x)$ on a $N=8192$ resolution grid together with their solution at time one $u(t=1, x)$. The data is split into training ($1448$ points), validation ($200$ points) and test ($400$ points) sets. We train for $100$ epochs on the $32$-fold downsampled resolution grid ($256$ DoFs instead of $8192$). The last layer $W \in \mathbb{R}^{H \times 1}$ is shared across spatial dimensions; consequently, a batch of size $B$ on a grid of $256$ points yields an effective batch of $256B$ for the closed-form update. We report the mean squared error (MSE) on the whole $N=8192$ grid (super-resolution evaluation).

\begin{figure}[h]
\begin{center}
\includegraphics[width=5.3in]{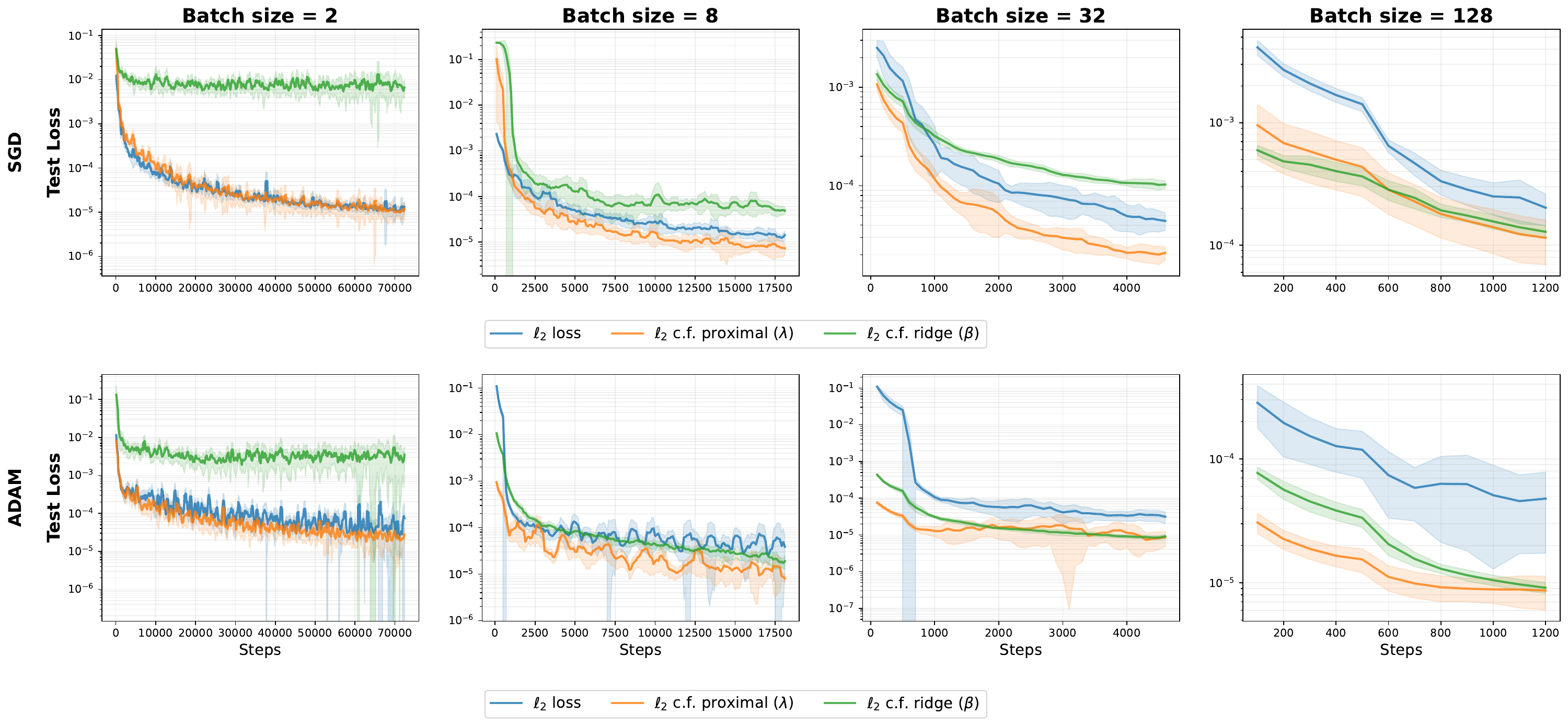}
\end{center}
\caption{\textbf{FNO results}. X-axis is the number of iterations, Y-axis is a test set mean squared error (MSE), columns represents different batch sizes. Different colors indicate different methods. We use a rolling average with window size $5$ to smooth the curves.}
\label{fig:regression_results}
\end{figure}

Results are shown in Figure~\ref{fig:regression_results}. Our approach ``\emph{$\ell_2$ c.f.~proximal ($\lambda)$}'' remains effective across batch sizes and optimizers. The ridge baseline ``\emph{$\ell_2$ c.f.~ridge ($\beta)$}'' fails to converge to a reasonable test loss at the smallest batch size (2) and is consistently outperformed by our method ``\emph{$\ell_2$ c.f.~proximal ($\lambda)$}'' at small batch sizes (8, 32) under SGD, while having similar performance only at larger batch sizes and with the Adam optimizer (on batch sizes 8, 32, 128). We conjecture that Adam’s per-coordinate step sizes partly compensate for a suboptimal last layer, which is why the gap to the baseline is smaller under Adam than under SGD.

Note that, as discussed above, the effective batch seen by the last layer is the nominal batch size multiplied by the number of spatial degrees of freedom (256), thus even a nominal batch size of $2$ corresponds to an effective batch size of $512$. This mirrors our findings on QM9: as the batch size grows, the mini-batch objective~(\ref{eq:l2_loss_batch}) increasingly approximates the full-dataset setting~(\ref{eq:old-loss}), diminishing the advantage of the proximal term.

\paragraph{Deep Feature Instrumental Variable (DFIV) regression.}

We evaluate our method in a causal inference setting using the two-stage DFIV framework~\citep{xu_learning_2020}; see~\cref{sec_app:dfiv} for background and details. We adapt DFIV to the minibatch setting and apply our proximal closed-form variant, which we call ``\emph{DFIV Proximal}''. 
For evaluation, we use two strategies. The first follows~\citep{xu_learning_2020} and re-estimates first-stage and second-stage last layers with ridge regression (we use $0.01$ coefficient for this) on the whole training set. The second strategy uses the current estimates of the last layers.

The results are given in~\cref{fig:causality_results_new}. For small batch sizes, our method outperforms standard DFIV, while achieving comparable performance in the large-batch regime. Notably, the online and re-estimated evaluations yield nearly identical results for our method, eliminating the need for a costly re-estimation pass over the full training set.


\begin{figure}[h]
\begin{center}
\includegraphics[width=5.3in]{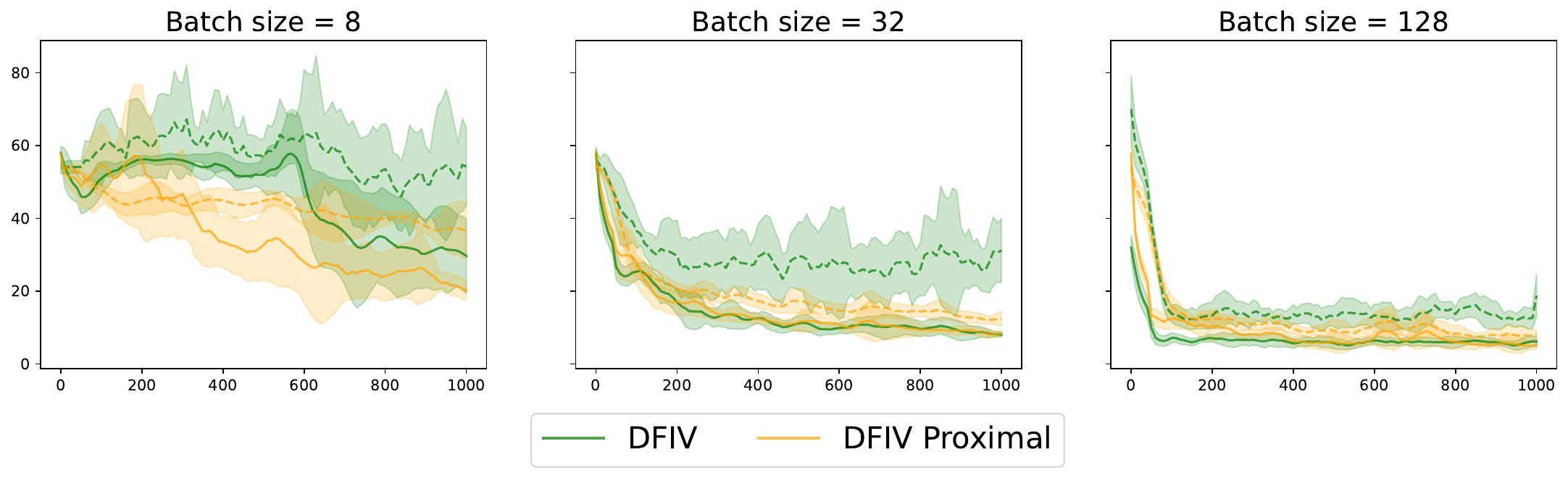}
\end{center}
\caption{\textbf{DFIV results}. X-axis is the number of iterations, Y-axis is a test set MSE. Each column corresponds to a different batch size. Different colors indicate different methods. Solid lines use the last layer re-estimated on the entire training set, while dashed lines use current last layer estimates. We use a rolling average with window size $5$ to smooth the curves.}
\label{fig:causality_results_new}
\end{figure}

\paragraph{Further results.} In \cref{sec_app:additional_results,sec_app:wall_clock,sec_app:wall_clock_synthetnic_qm9} we conduct additional experiments, notably we record timings (\cref{sec_app:wall_clock,sec_app:wall_clock_synthetnic_qm9}), memory footprint (\cref{sec_app:wall_clock_synthetnic_qm9}) and classification tasks (\cref{sec_app:classification_results}). For such tasks, ``\emph{$\ell_2$ loss}'' usually underperforms ``\emph{Cross Entropy}'' but, surprisingly, ``\emph{$\ell_2$ c.f.~proximal ($\lambda)$}'' achieves comparable performance to ``\emph{Cross Entropy}'' on datasets with small to moderate number of classes. The closed-form update adds $O(Bd^2 + d^3)$ compute and $O(d^2)$ memory per step, where $B$ is the batch size and $d$ the last-layer feature dimension, independent of the backbone size. In practice, this amounts to a roughly 10--15\% slowdown for large batches, but the overhead grows when the batch size is small and the last-layer dimension is large (up to $2.4\times$ for $B=32$, $d=4096$ on ResNet-18; see Appendix~\ref{sec_app:wall_clock}).

\section{Conclusion}
\label{sec:conclusion}

We proposed leveraging closed-form optimal solutions for the last layer of neural networks under squared loss as an integral part of optimization. 

Theoretically, we showed that the closed-form approach is well-founded in the neural tangent kernel regime, and bounded its superiority over gradient steps in a one-step analysis.

On regression tasks—including quantum chemistry (QM9), PDE operator learning (FNO/Burgers), and causal inference (DFIV)—our proximal closed-form approach consistently accelerates training and matches or outperforms standard gradient-based optimization of the squared loss. Compared to naive closed-form ridge regression, our proximal formulation is substantially more robust, particularly in the small-batch regime. The DFIV results are especially compelling: the closed-form solution is naturally motivated by the two-stage structure, and our method eliminates the need for costly re-estimation on the full training set.

When naively applied to classification, our method based on the squared loss achieves performance comparable to SGD on the cross entropy loss for tasks with a small to moderate number of classes. We observe it underperforms for tasks with a large number of classes. Indeed, a limitation of our work is that it only applies to the squared loss. Adapting it to cross entropy presents a promising area of future work, so that it may be successfully applied to more scenarios.

In addition, in future work we aim to apply our method to larger-scale two-stage problems such as offline reinforcement learning~\citep{chen_instrumental_2022} and proxy variable regression~\citep{xu_deep_2021}. We also aim to use a dimension-dependent regularization vector $\lambda$, and schedules for it that adapt over the course of training, which could mirror the strategies from adaptive optimizers such as Adam.





\section{Acknowledgments}
NDC and PH gratefully acknowledge co-funding by the European Union (ERC, ANUBIS, 101123955). Views and
opinions expressed are however those of the author(s) only and do not necessarily reflect those of the European Unionor the European Research Council. Neither the European Union nor the granting authority can be held responsible for them. NDC is supported by the Fonds National de la Recherche, Luxembourg, Project 17917615. PH is supported by the DFG through Project HE 7114/6-1 in SPP2298/2. PH is a member of the Machine Learning Cluster of Excellence, funded by the Deutsche Forschungsgemeinschaft (DFG, German Research Foundation) under Germany’s Excellence Strategy – EXC number 2064/1 – Project number 390727645. NDC and PH also gratefully acknowledge the German Federal Ministry of Education and Research (BMBF) through the Tübingen AI Center (FKZ:01IS18039A); and funds from the Ministry of Science, Research and Arts of the State of Baden-Württemberg.

\bibliographystyle{plain}
\bibliography{bibliography}

\appendix

\newpage

\startcontents[sections]
\printcontents[sections]{l}{1}{\setcounter{tocdepth}{2}}
\counterwithin{figure}{section}
\counterwithin{table}{section}

\newpage

\section{Derivations for the Closed-Form Solutions}
In this appendix we derive the closed-form formulas for $W^\star(\theta)$ (\cref{eq:closed-form-w}) and $W^\star_{\batch_t,W_t}(\theta)$ (\cref{eq:closed-form-proximal}).
\subsection{Ridge Loss}\label{sec_app:ridge_derivation}
Recall \cref{eq:old-loss}:
\begin{equation}
    \begin{aligned}
        \loss(W,\theta) &= \sum_{i=1}^n\|y_i - W\phi_\theta(x_i)\|_2^2 + \beta\|W\|^2_F \\
        &=\|Y-W\phi_\theta(X)\|^2_F + \beta\|W\|^2_F
    \end{aligned}
\end{equation}
which we rewrote in matrix form, $\phi_\theta(X) \in \R^{d\times n}$, $Y\in \R^{o\times n}$, $W\in \R^{o\times d}$. Differentiating with respect to $W$ we get
\begin{equation}
    \begin{aligned}
        \nabla_W \loss(W,\theta) &= -2(Y-W\phi_\theta(X))\phi_\theta(X)^\top +2\beta W \\
        &= 2W\left(\phi_\theta(X)\phi_\theta(X)^\top +\beta I\right) - 2Y\phi_\theta(X)^\top.
    \end{aligned}
\end{equation}
Solving for $W^\star$ such that $\nabla_W\loss(W^\star,\theta) = 0$ we get the unique solution
\begin{equation}
    W^\star = Y\phi_\theta(X)^\top\left(\phi_\theta(X)\phi_\theta(X)^\top +\beta I\right)\inv.
\end{equation}

\subsection{Proximal Loss}\label{sec_app:proximal_derivation}
Recall \cref{eq:l2_prox}:
\begin{equation}
    \begin{aligned}
        \loss^{\prox}_{\batch_t,  W_t} (W, \theta) &= \sum_{(x_i,y_i)\in \batch_t} \|y_i - W\phi_\theta(x_i)\|_2^2 + \lambda\|W -  W_t\|_F^2 \\
        &=\|Y_t-W\phi_\theta(X_t)\|^2_F + \lambda\|W-W_t\|^2_F.
    \end{aligned}
\end{equation}
in matrix form, where $X_t$ and $Y_t$ correspond to the matrix of inputs and outputs respectively of batch $\batch_t$. As in \cref{sec_app:ridge_derivation}, we compute the gradient with respect to $W$:
\begin{equation}
    \begin{aligned}
        \nabla_W \loss^{\prox}_{\batch_t,  W_t}(W,\theta) &= -2(Y_t-W\phi_\theta(X_t))\phi_\theta(X_t)^\top +2\lambda (W-W_t) \\
        &= 2W\left(\phi_\theta(X_t)\phi_\theta(X_t)^\top +\lambda I\right) - 2Y_t\phi_\theta(X_t)^\top -2\lambda W_t.
    \end{aligned}
\end{equation}
Solving for $W^\star$ such that $\nabla_W\loss^{\prox}_{\batch_t,  W_t} (W^\star, \theta) = 0$ we get the unique solution
\begin{equation}
    W^\star = \left(Y_t\phi_\theta(X_t)^\top+  \lambda W_t\right)\left(\phi_\theta(X_t)\phi_\theta(X_t)^\top+\lambda I\right)\inv.
\end{equation}

\subsection{Generalized Ridge Loss}\label{sec_app:gen_derivation}
\cite{hui2021evaluation} proposes to include additional hyperparameters within the squared loss objective in order to improve the performance on classification tasks. Namely, if $c$ is the correct class for datapoint $x$, consider the loss
\begin{equation}
    \ell(W,\theta) = k \cdot ([W\phi_{\theta}(x)]_c - m)^2 + \sum_{j=1, j \neq c}^{C} [W\phi_{\theta}(x)]_j^2
\end{equation}
for some hyperparameters $k,m>0$. $\loss^{\operatorname{gen}}(W,\theta)$ is then given by the sum of the $\ell(W,\theta)$ for all training datapoints $x_i$ plus a regularizer. In this appendix we use a ridge regularizer $\beta\|W\|^2_F$, though we may use a proximal regularizer instead.

Let $W_{c\cdot}$, $Y_{c\cdot}$ be the $c$\textsuperscript{th} row of the last layer, labels respectively, which we view both as column vectors. Let $z^{(c)}\in \R^n$ be the vector with entry $z^{(c)}_i = \sqrt{k}$ if $Y_{ci} = 1$, $z^{(c)}_i = 1$ otherwise. That is, $z^{(c)} = \mathbf{1} + (\sqrt{k}-1)Y_{c\cdot}$, where $\mathbf{1}\in\R^n$ is the vector with ones in all entries. Finally, let $\odot$ denote the Hadamard product. We define the loss `restricted' to the $c$\textsuperscript{th} class as:

\begin{equation}
    \loss^{\operatorname{gen}}_c(W_{c\cdot},\theta) = \left[ z^{(c)} \odot \left( m Y_{c\cdot} -  \phi_\theta(X)^\top W_{c\cdot} \right) \right]^\top
    \left[ z^{(c)} \odot \left( m Y_{c\cdot} - \phi_\theta(X)^\top W_{c\cdot} \right) \right]
    + \beta \, W_{c\cdot}^\top W_{c\cdot}.
\end{equation}

$\loss^{\operatorname{gen}}(W,\theta)$ is the sum over the $\loss^{\operatorname{gen}}_c(W_{c\cdot},\theta)$. Differentiating only with respect to the $c$\textsuperscript{th} row of $W$,

\begin{equation}
    \begin{aligned}
        \nabla_{W_{c\cdot}}\loss^{\operatorname{gen}}(W,\theta) &= \nabla_{W_{c\cdot}}\loss^{\operatorname{gen}}_c(W_{c\cdot},\theta) \\
        &= -2 \phi_\theta(X)\left( (z^{(c)})^2 \odot \left( m Y_{c\cdot} - \phi_\theta(X)^\top W_{c\cdot} \right) \right) + 2\beta W_{c\cdot} \\
        &= 2\left(\phi_\theta(X)\left((z^{(c)})^2 \odot \phi_\theta(X)\right)^\top +\beta I\right)W_{c\cdot} - 2\phi_\theta(X) \left((z^{(c)})^2 \odot mY_{c\cdot}\right) \\
        &= 2\left(\left(z^{(c)}\odot\phi_\theta(X)\right)\left(z^{(c)}\odot \phi_\theta(X)\right)^\top +\beta I\right)W_{c\cdot} \\
        &\quad- 2 \left(z^{(c)}\odot\phi_\theta(X)\right) \left(z^{(c)} \odot mY_{c\cdot}\right) \\
        &= 2\left(\phi^{(c)}_\theta(X)\phi^{(c)}_\theta(X)^\top +\beta I\right)W_{c\cdot} - 2\phi^{(c)}_\theta(X) Y^{(c)}_{c\cdot}
    \end{aligned}
\end{equation}

where the square over $z^{(c)}$ is taken component-wise, $\phi^{(c)}_\theta(X):= z^{(c)}\odot\phi_\theta(X)$ and $Y^{(c)}_{c\cdot} := z^{(c)} \odot mY_{c\cdot}$.

Solving for $W^\star_{c\cdot}$ such that $\nabla_{W_{c\cdot}}\loss^{\operatorname{gen}}_c(W,\theta)=0$ we get
\begin{equation}\label{eq:sol_gen_loss}
    \begin{aligned}
        W_{c\cdot} = \left( \phi_\theta^{(c)}(X) \phi^{(c)}_\theta(X)^\top + \beta I \right)^{-1}
     \phi^{(c)}_\theta(X) Y^{(c)}_{c\cdot}.
    \end{aligned}
\end{equation}
\Cref{eq:sol_gen_loss} presents a computational bottleneck for large number of classes, since it requires a different matrix inversion per class.

\section{Kalman Filter Interpretation}\label{sec_app:kalman}
We can interpret the updates on the last layer in \cref{eq:final-sgd-update} as a Kalman filter under several simplifying assumptions. We take the Bayesian point of view, where we treat the last layer $W_t$ at time $t$ given the feature parameters $\theta_t$ as a random variable. See also~\citep{titsias2024kalman} for a similar discussion.

First, we assume that the model fits the data perfectly during optimization, so we have the likelihood
\begin{equation}\label{eq:kalman-likelihood}
    p(y_i \mid x_i, W_t, \theta_t) = \mathcal N(y_i\mid W_t\phi_{\theta_t}(x_i), \sigma^2_YI)
\end{equation}
where $I$ is the identity matrix and $\sigma^2_Y$ is some hyperparameter controlling the variance of the outputs and $(x_i,y_i)\in\batch_t$.

Next, we assume $W_t$ evolves like a random walk with Gaussian steps, as the parameters evolve through time $\theta_t$, so
\begin{equation}\label{eq:kalman-dynamics}
    p(W_{t+1}\mid W_t) = \mathcal N(W_{t+1}\mid W_t, \sigma^2_WI)
\end{equation}
where $\sigma^2_W$ is some hyperparameter controlling the variance of the steps.

\Cref{eq:kalman-likelihood,eq:kalman-dynamics} provide us a way to update our belief about $W_t$ in closed-form through time. Namely our belief about $W_t$ given our observations $\batch_s$ for $s< t$ will be a Gaussian distribution $\mathcal N( W_t, \Sigma_t)$, obtained by \emph{Kalman filtering} \citep{sarkka_bayesian_2013}. Here, note that $\Sigma_t$ will be a $od \times od$ matrix, i.e.~will be squared times the number of parameters in the last layer. For large last layers, this can be intensive to store and manipulate.

Instead, we propose an additional simplifying approximation: we ignore the covariance $\Sigma_t$ at each step, and collapse our belief over $W_{t+1}$ to its maximum-a-posteriori estimate. The resulting update on our point estimate of the last layer is given by:
\begin{equation}
    \begin{aligned}
         W_{t+1} = \argmin_{W\in \R^{o\times d}}&-\sum_{(x_i,y_i)\in\batch_t} \log p(y_i \mid x_i,W,\theta_t)-\log p(W\mid  W_t) \\
        = \argmin_{W\in \R^{o\times d}} &-\sum_{(x_i,y_i)\in\batch_t}\log \mathcal N(y_i\mid W\phi_{\theta_t}(x_i), \sigma^2_YI)-\log \mathcal N(W\mid  W_t, \sigma^2_WI) \\
        = \argmin_{W\in \R^{o\times d}} &\sum_{(x_i,y_i)\in \batch_t} \frac{1}{2\sigma_Y^2}\|y_i - W\phi_{\theta_t}(x_i)\|_2^2 + \frac{1}{2\sigma_W^2}\|W -  W_t\|_F^2 \\
        = W^\star_{\mathcal B_t, W_t}(&\theta)
    \end{aligned}
\end{equation}
with $\lambda= \frac{\sigma_Y^2}{\sigma_W^2}$ in \cref{eq:closed-form-proximal}. That is, such approximate Bayesian updates recovers precisely the minimum of the proximal loss \cref{eq:closed-form-proximal}, leading to the updates on $ W_t$ as in \cref{eq:new-sgd-update,eq:final-sgd-update}. 

\section{Detailed Theoretical Analyses}\label{sec_app:theory}
This appendix contains the analysis and proofs for \cref{sec:theory}.

\subsection{Analysis of the Modified Loss}\label{sec_app:modified_loss}
In this appendix we prove \cref{thm:non-convex}.

To begin, consider the case $\beta>0$. Spelling out the expression for $\loss_\fsp^\star$ in matrix form,

\begin{equation}
    \begin{aligned}
        \loss^\star_\fsp(\phi) &:= \sum_{i=1}^n\|y_i - W^\star_\fsp(\phi)\phi(x_i)\|_2^2+\beta\|W^\star_\fsp(\phi)\|^2_F \\
        &= \|Y- W^\star_\fsp(\phi) \phi(X)\|_F^2+\beta\|W^\star_\fsp(\phi)\|^2_F \\
        &= \tr\left((Y- W^\star_\fsp(\phi) \phi(X))^\top (Y- W^\star_\fsp(\phi) \phi(X))\right)+\beta\tr\left(W^\star_\fsp(\phi)^\top W^\star_\fsp(\phi)\right)
    \end{aligned}
\end{equation}
where recall, $\phi(X)\in \R^{d\times n}$, $Y\in \R^{o\times n}$ and $W^\star_\fsp(\phi) := Y\phi(X)^\top\left(\phi(X)\phi(X)^\top+\beta I\right)\inv \in \R^{o\times d}$.

The derivative of $\loss_\fsp^\star$ at $\phi$ is a linear map $\D \loss_\fsp^\star(\phi)\colon\fsp \to \R$. Just as in \cref{thm:gradient-equiv}, to calculate $\D \loss_\fsp^\star(\phi)$ we do not need to differentiate $W^\star_\fsp(\phi)$ with respect to $\phi$, and may treat it as constant instead. So for $\psi\in\fsp$,
\begin{equation}
    \D \loss_\fsp^\star(\phi) [\psi] = \tr\big(\underbrace{-2(Y- W^\star_\fsp(\phi) \phi(X))^\top W^\star_\fsp(\phi)}_{=:\nabla\loss^\star_\fsp(\phi)^\top} \psi(X)\big).
\end{equation}
The definition of the gradient $\nabla\loss^\star_\fsp(\phi) \in \R^{d\times n}$ is just a stand-alone notation used for convenience as -- without an inner product on $\fsp$ -- we do not have a well defined notion of gradients for the functional $\loss^\star_\fsp$. Importantly, note that $\phi^\star$ is a critical point of $\loss^\star_\fsp$ if and only if $\D \loss_\fsp^\star(\phi^\star)=0$, i.e.~if and only if $\D \loss_\fsp^\star(\phi^\star) [\psi] =0$ for all $\psi\in \fsp$, i.e.~if and only if $\nabla\loss^\star_\fsp(\phi^\star) = 0$.

Plugging in the expression for $W^\star_{\fsp}(\phi)$ in the definition of $\nabla\loss^\star_\fsp(\phi)$ and writing $\Phi:=\phi(X)\in \R^{d\times n}$, $\Phi^+_\beta := \Phi^\top(\Phi\Phi^\top +\beta I)\inv \in \R^{n\times d}$, we get
\begin{equation}\label{eq:funct-grad}
    \begin{aligned}
    \nabla\loss^\star_\fsp(\phi)&= 2W^\star_{\fsp}(\phi)^\top (W^\star_{\fsp}(\phi)\phi(X) - Y) \\
    &= 2(\Phi\Phi^\top +\beta I)\inv\Phi Y^\top Y(\Phi^\top(\Phi\Phi^\top +\beta I)\inv\Phi-I) \\
    &= 2(\Phi^+_\beta)^\top Y^\top Y(\Phi^+_\beta\Phi-I).
    \end{aligned}
\end{equation}
Define
\begin{equation}
    Y_\star = Y\Phi^+\Phi,\quad Y_\perp = Y(I-\Phi^+\Phi).
\end{equation}
where $\Phi^+ \in \R^{n\times d}$ is the Moore-Penrose pseudo-inverse of $\Phi$. In particular
\begin{equation}
    Y = Y_\star + Y_\perp.
\end{equation}
$Y_\star$ should be thought of the part of the outputs `attainable' by the features $\Phi$, and $Y_\perp$ the part of the outputs `unattainable'. To see this, take a singular value decomposition (SVD) of $\Phi$ of the form $\Phi = U\Sigma V^\top$ where $U\in \R^{d\times d}$ has orthonormal columns, $\Sigma \in \R^{d\times d}$ is diagonal with the first $r := \rank \Phi$ diagonal entries being non-zero and $V\in \R^{n \times d}$ has orthonormal columns. Moreover write $\Vsp_\star \leq \R^n$ for the subspace spanned by the rows of $\Phi$, or equivalently the first $r$ columns of $V$, the space of `attainable' outputs. Further let $\Vsp_\perp\leq \R^n$ its orthogonal complement, the space of `unattainable' outputs. We have $\Phi^+\Phi = V\Lambda V^\top$ the orthogonal projection onto $\Vsp_\star$, where $\Lambda_{ij} =1$ for $i=j\leq r$, $\Lambda_{ij} =0$ otherwise. Therefore the rows of $Y_\star$ are the orthogonal projection of those of $Y$ onto $\Vsp_\star$, and those of $Y_\perp$ are the orthogonal projection onto $\Vsp_\perp$.

Now note that
\begin{equation}
    \begin{aligned}
        (\Phi^+_\beta)^\top Y^\top &= (\Phi\Phi^\top +\beta I)\inv\Phi Y^\top\\
        &= U(\Sigma^2 +\beta I)\inv \Sigma V^\top Y^\top \\
        &= U(\Sigma^2 +\beta I)\inv \Sigma V^\top Y_\star^\top \\
        &= (\Phi^+_\beta)^\top Y_\star^\top.
    \end{aligned}
\end{equation}
So from \cref{eq:funct-grad},
\begin{equation}\label{eq:funct-grad-compressed}
    \nabla\loss^\star_\fsp(\phi) = 2(\Phi^+_\beta)^\top Y^\top_\star Y(\Phi^+_\beta\Phi-I).
\end{equation}
We see for example that $\phi = 0$, or more generally whenever the rows of $\Phi= \phi(X)$ are orthogonal to the rows of $Y$, we have $Y_\star = 0$ and $Y_\perp = Y$, so is a critical point of $\loss^\star_\fsp$, even though it is not a global minimizer, assuming $Y\neq 0$.

When $Y\neq 0$, $\loss^\star_\fsp$ admits critical points which are not global minima thus it is not convex, so this concludes the proof of \cref{thm:non-convex} for $\beta>0$.

\begin{remark}
    To minimize $\loss^\star_\fsp(\phi)$ in the case $\beta>0$, one must express the rows of $Y$ (all, or as many as possible) with the span of the rows of $\Phi$, while at the same time taking $\Sigma_{ii}\to 0$ for all $i$ or $\Sigma_{ii}\to\infty$ for some $i$, to ensure that $\|W^\star_\fsp(\phi)\|_F^2 = \|Y V \Sigma(\Sigma^2 +\beta I)\inv U^\top\|_F^2 \to 0$. The loss does not attain a global minimum.
\end{remark}

For $\beta = 0$, the minimizer $W^\star_\fsp(\phi)$ of \cref{eq:loss-fsp} is not unique, so we make the natural choice $W^\star_\fsp(\phi) = Y\Phi^+$. The Moore-Penrose pseudo-inverse operation is not continuous and hence not differentiable. But it is differentiable on any manifold of constant rank $\Phi$, and hence so is $\loss^\star_\fsp$. Specifically (c.f.~\cref{eq:funct-grad,eq:funct-grad-compressed})
\begin{equation}
    \nabla\loss^\star_\fsp(\phi) = 2(\Phi^+)^\top Y^\top Y(\Phi^+\Phi-I) = -2(\Phi^+)^\top Y^\top_\star Y_\perp.
\end{equation}

As for $\beta>0$, any $\phi$ where the rows of $\Phi$ are orthogonal to the rows of $Y$ is a critical point of $\loss^\star_\fsp$. When $Y\neq 0$, this is not a global minimizer. \cref{thm:non-convex} thus also follows in the case $\beta=0$.

\begin{remark}\label{rk:minimizers}
    When $d \geq \rank Y$ and the $x_i$ are distinct, that is when the outputs $Y$ are expressible through the features $\phi(X)$, the minimizers correspond exactly to the points $\phi$ such that $Y_\perp = 0$ and $Y_\star = Y$.
\end{remark}

\begin{remark}
    A result such as \cref{thm:gradient-equiv} cannot be extended to second derivatives, otherwise we could differentiate $\loss^\star_\fsp$ twice by keeping $W^\star_\fsp(\phi)$ constant, and would obtain that the Hessian of $\loss^\star_\fsp$ is positive semi-definite since so is the one of the squared loss. But this is impossible since we showed that $\loss^\star_\fsp$ is not convex.
\end{remark}

\subsection{Neural Tangent Kernel Infinite Width Limit}\label{sec_app:ntk}
In this appendix we provide a self-contained overview of the neural tangent kernel (NTK) limit, based on \cite{jacot_neural_2018}, and then prove \cref{thm:ntk}.

$\phi_\theta$ is assumed to be a neural network, $\theta$ are its parameters consisting of weights $W^{(\ell)}$ and biases $b^{(\ell)}$, such that $\phi_\theta(x) = \alpha^{(L)}_\theta(x)$, the pre-activations $\tilde\alpha^{(\ell)}_\theta(x) \colon \R^{d_0}\to \R^{d_\ell}$ the activations $\alpha^{(\ell)}_\theta(x) \colon \R^{d_0}\to \R^{d_\ell}$, $d_L = d$ and
\begin{equation}\label{eq:ntk-architecture}
    \begin{aligned}
        \alpha^{(0)}_\theta(x) & =x \\
        \tilde{\alpha}^{(\ell+1)}_\theta(x) & =\frac{1}{\sqrt{d_{\ell}}} W^{(\ell)} \alpha^{(\ell)}_\theta(x)+b^{(\ell)} \\
        \alpha^{(\ell)}_\theta(x) & =\sigma\left(\tilde{\alpha}^{(\ell)}_\theta(x)\right),
    \end{aligned}
\end{equation}
where $\sigma\colon \R \to \R$ is a twice differentiable non-linearity function with bounded second derivative, applied element-wise. The parameters $\theta$ are initialized at time $t=0$ with $W_{ij}^{(\ell)} \sim \mathcal N(0, 1)$, $b_j^{(\ell)} \sim \mathcal N(0,1)$ which, combined with the pre-multiplicative factors $1/\sqrt{d_\ell}$ in \cref{eq:ntk-architecture}, corresponds to the LeCun initialization (see \cref{subsec:numerical_considerations}).

We then consider gradient flow on the loss $\loss^\star_\fsp$:
\begin{equation}\label{eq:gradient-flow}
    \frac{\dv\theta}{\dv t} = -\nabla_{\theta}\loss^\star_\fsp(\theta).
\end{equation}


Consider the limit $d_1,\dots, d_{L-1} \to \infty$ sequentially, that is we first take $d_1\to \infty$, then $d_2\to\infty$, etc.~up to $d_{L-1}\to\infty$. In this limit, under some integrability assumption (see \cref{rk:stochastic-bded}), the dynamics of the function $\phi=\phi_\theta$ under \cref{eq:gradient-flow} are given by kernel gradient descent: for $x\in \R^{d_0}$, $\beta =0$,
\begin{equation}\label{eq:kernel-gd}
    \begin{aligned}
        \frac{\dv\phi(x)}{\dv t} &= -\sum_{i=1}^n k(x,x_i)
        \nabla \loss^\star_\fsp (\phi)_i \\
        &= 2(\Phi^+)^\top Y^\top_\star Y_\perp k(X,x)
    \end{aligned}
\end{equation}
where we used \cref{eq:funct-grad-compressed}, and $k$ is a positive semi-definite kernel $k\colon \R^{d_0}\times \R^{d_0} \to \R$, the \emph{neural tangent kernel} \citep[Theorem 1 \& 2]{jacot_neural_2018}, $k(X,x) \in \R^n$. Whenever this kernel is positive definite (that is for all $x_1,\dots,x_N\in \R^{d_0}$, the matrix $(k(x_i,x_j))_{i,j}$ is positive definite; see for example \citep[Proposition 2]{jacot_neural_2018}) we know that, as $t\to\infty$, $\phi$ will converge pointwise to a critical point $\phi^\star$ of $\loss^\star_\fsp$. The goal of \cref{thm:ntk} is to show that, almost surely in the initialization, $\phi^\star$ will be a global minimum of $\loss^\star_\fsp$. In other words, under the assumption that $d \geq \rank Y$ and that the $x_i$ are distinct, we want to show $Y_\star \to Y$ and $Y_\perp \to 0$ as $t\to \infty$ (see \cref{rk:minimizers}).

From \cref{eq:kernel-gd}, we see that

\begin{equation}
    \frac{\dv\Phi}{\dv t} = 2(\Phi^+)^\top Y^\top_\star Y_\perp K 
\end{equation}
where $K := k(X,X)\in \R^{n\times n}$.

First, consider the case where $d\leq n$. At initialization, the infinite width neural network is a Gaussian process \citep[Proposition 1]{jacot_neural_2018}. If the corresponding \emph{neural Gaussian process kernel} (NGPK) is positive definite (see for example \citep[Theorem 4.5]{gao_wide_2023}) then the distribution of the columns of $\Phi$ corresponding to distinct $x_i$ follow a non-degenerate Gaussian at initialization ($t=0$). So, since $d\leq n$, $\Phi$ has full row rank $d$ at $t=0$ almost surely. Thus, in that case, $\Phi^+ = \Phi^\top(\Phi\Phi^\top)\inv$. So

\begin{equation}
\begin{aligned}
    \frac{\dv Y_\star}{\dv t} &= \frac{\dv}{\dv t} \left(Y \Phi^\top (\Phi\Phi^\top)\inv\Phi \right) \\
    &= Y \frac{\dv \Phi^\top}{\dv t} (\Phi\Phi^\top)\inv\Phi - Y \Phi^\top (\Phi\Phi^\top)\inv\frac{\dv \Phi}{\dv t}\Phi^\top(\Phi\Phi^\top)\inv\Phi \\
    &\quad- Y\Phi^\top (\Phi\Phi^\top)\inv\Phi\frac{\dv \Phi^\top}{\dv t}(\Phi\Phi^\top)\inv\Phi + Y\Phi^\top (\Phi\Phi^\top)\inv\frac{\dv \Phi}{\dv t} \\
    &= 2Y KY_\perp^\top Y_\star\Phi^\top(\Phi\Phi^\top)^{-2} \Phi - 2Y\Phi^\top(\Phi\Phi^\top)^{-2} \Phi Y_\star^\top Y_\perp K \Phi^\top (\Phi\Phi^\top)\inv \Phi \\
    &\quad - 2Y\Phi^\top(\Phi\Phi^\top)\inv\Phi KY_\perp^\top Y_\star\Phi^\top (\Phi\Phi^\top)^{-2} \Phi+ 2Y \Phi^\top(\Phi\Phi^\top)^{-2} \Phi Y_\star^\top Y_\perp K.
\end{aligned}
\end{equation}
Hence
\begin{equation}
    \begin{aligned}
        \frac{\dv Y_\star}{\dv t}Y_\star^\top &= 2Y KY_\perp^\top Y_\star\Phi^\top(\Phi\Phi^\top)^{-2} \Phi Y_\star^\top - 2Y\Phi^\top(\Phi\cancel{\Phi^\top)^{-2} \Phi Y_\star^\top Y_\perp K \Phi^\top (\Phi}\Phi^\top)\inv \Phi Y_\star^\top \\
        &\quad - 2Y\Phi^\top(\Phi\Phi^\top)\inv\Phi KY_\perp^\top Y_\star\Phi^\top (\Phi\Phi^\top)^{-2} \Phi Y_\star^\top+ 2Y \cancel{\Phi^\top(\Phi\Phi^\top)^{-2} \Phi Y_\star^\top Y_\perp K} Y_\star^\top \\
        &= 2Y(I-\Phi^\top(\Phi\Phi^\top)\inv\Phi) KY_\perp^\top Y_\star\Phi^\top(\Phi\Phi^\top)^{-2} \Phi Y_\star^\top \\
        &= 2\underbrace{\left(Y_\perp KY_\perp^\top\right)}_{=:A} \underbrace{\left(Y_\star\Phi^\top(\Phi\Phi^\top)^{-2} \Phi Y_\star^\top\right)}_{=:B}.
    \end{aligned}
\end{equation}
So
\begin{equation}\label{eq:evolution_y*}
    \frac{\dv}{\dv t}\|Y_\star\|_F^2 =2\tr\left( \frac{\dv Y_\star}{\dv t}Y_\star^\top\right) = 2\tr(AB)
\end{equation}

At $t=0$, $\dim \Vsp_\star =\rank \Phi = d$ almost surely. Since by assumption $d \geq \rank Y$, projecting the rows of $Y$ onto $\Vsp_\star$ we get that $\rank Y_\star = \rank Y$ almost surely at $t=0$. Moreover $B = Y_\star \Phi^\top (\Phi\Phi^\top)^{-2} \Phi Y_\star^\top = Y_\star V\Sigma^{-2}V^\top Y_\star^\top$ in the SVD notation from before, and $\Sigma_{ii}>0$ for all $i$ almost surely. So $B$ is positive definite almost surely. $A$ is positive semi-definite so $\tr(AB)\geq 0$, and hence by \cref{eq:evolution_y*}, $\|Y_\star\|_F^2$ is non-decreasing through time. It converges when $A = 0$, i.e.~$Y_\perp =0$, which corresponds to global minima of $\loss^\star_\fsp$. $Y_\star$ is guaranteed to converge when the NTK is positive definite.

We see in addition that, since the NTK is the sum of the NGPK with some other positive semi-definite kernel \citep[Theorem 1]{jacot_neural_2018}, the NGPK being positive definite implies that the NTK is too.

For the case $n< d$, the optimization becomes trivial almost surely. Indeed, at $t=0$, $\dim \Vsp_\star =\rank \Phi = n$ almost surely. In other words $\Vsp_\star = \R^n$ and $\Vsp_\perp = \{0\}$, and thus $Y_\star =  Y$ and $Y_\perp = 0$ almost surely at $t=0$.

\begin{remark}\label{rk:stochastic-bded}
    To show that the infinite width limit leads to kernel gradient descent as in \cref{eq:kernel-gd}, one needs to assume that, for any $T>0$, the random variable (over the random initialization on $\theta$) given by the integral
    \begin{equation}
        \int_0^T \left\|\nabla\loss^\star_\fsp(\phi_\theta)\right\|_F \mathrm{d}t
    \end{equation}
    stays stochastically bounded as $d_1,\dots,d_{L-1}\to \infty$ sequentially \citep[Theorem 2]{jacot_neural_2018}.
\end{remark}

\subsection{Analysis of One-Step Performance}\label{sec_app:one-step}
In this appendix we prove \cref{thm:one-step}.

Taking a second order Taylor expansion of the loss around $W$:
\begin{equation}\label{eq:taylor-loss}
    \loss(W_{GD},\theta) = \loss(W,\theta)-\eta \|\nabla_W \loss(W,\theta)\|^2_F + \frac{\eta^2}{2}H_W \loss(W,\theta)\left[\nabla_W \loss(W,\theta),\nabla_W \loss(W,\theta)\right]
\end{equation}
which is an exact equality since $\loss$ is quadratic in $W$. \Cref{eq:taylor-loss} is minimized at the learning rate
\begin{equation}\label{eq:optimal-lr}
    \eta^* = \frac{\|\nabla_W \loss(W,\theta)\|^2_F}{H_W \loss(W,\theta)\left[\nabla_W \loss(W,\theta),\nabla_W \loss(W,\theta)\right]}.
\end{equation}
Note that the denominator of \cref{eq:optimal-lr} is non-zero since $W\not\in \argmin_{\tilde W} \loss(\tilde W,\theta)$. Let $\phi_\theta(X) = U\Sigma V^\top$ be an SVD, with $U \in \R^{d\times r}$, $\Sigma\in \R^{r\times r}$, $V\in \R^{n \times r}$ where $r$ is the rank of $\phi_\theta(X)$. Plugging \cref{eq:optimal-lr} into \cref{eq:taylor-loss} we get that for any $\eta>0$
\begin{equation}\label{eq:gd-decrease}
    \begin{aligned}
        \loss(W,\theta) - \loss(W_{GD},\theta) &\leq \frac{\|\nabla_W \loss(W,\theta)\|^4_F}{2H_W \loss(W,\theta)\left[\nabla_W \loss(W,\theta),\nabla_W \loss(W,\theta)\right]} \\
        &= \frac{\|(W\phi_\theta(X)-Y)\phi_\theta(X)^\top\|^4_F}{\|(W\phi_\theta(X)-Y)\phi_\theta(X)^\top\phi_\theta(X)\|^2_F} \\
        &= \frac{\|(W\phi_\theta(X)-Y)V\Sigma U^\top\|^4_F}{\|(W\phi_\theta(X)-Y)V\Sigma^2V^\top\|^2_F} \\
        &= \frac{\|(W\phi_\theta(X)-Y)V\Sigma\|^4_F}{\|(W\phi_\theta(X)-Y)V\Sigma^2\|^2_F}.
    \end{aligned}
\end{equation}
Now let $\mathbb P$ be the categorical distribution over the singular values $\sigma$ of $\phi_\theta(X)$ given by
\begin{equation}\label{eq:categorical}
    \mathbb P(\sigma = \Sigma_{ii}) = \sum_{j:\Sigma_{jj}=\Sigma_{ii}}\frac{\|((W\phi_\theta(X)-Y)V)_{\cdot j}\|^2}{\|(W\phi_\theta(X)-Y)V\|_F^2}.
\end{equation}
The sum in \cref{eq:categorical} is over repeated singular values. We get
\begin{equation}\label{eq:prob-interp1}
    \begin{aligned}
        \|(W\phi_\theta(X)-Y)V\Sigma\|^2_F &= \|(W\phi_\theta(X)-Y)V\|^2_F \cdot \sum_{i=1}^r \frac{\|(W\phi_\theta(X)-Y)V)_{\cdot i}\|^2}{\|(W\phi_\theta(X)-Y)V\|^2_F}\Sigma_{ii}^2 \\
        &= \|(W\phi_\theta(X)-Y)V\|^2_F\cdot \E_{\mathbb P}[\sigma^2]
    \end{aligned}
\end{equation}
and similarly
\begin{equation}\label{eq:prob-interp2}
    \begin{aligned}
        \|(W\phi_\theta(X)-Y)V\Sigma^2\|^2_F &= \|(W\phi_\theta(X)-Y)V\|^2_F \cdot \sum_{i=1}^r \frac{\|(W\phi_\theta(X)-Y)V)_{\cdot i}\|^2}{\|(W\phi_\theta(X)-Y)V\|^2_F}\Sigma_{ii}^4 \\
        &= \|(W\phi_\theta(X)-Y)V\|^2_F\cdot \E_{\mathbb P}[\sigma^4].
    \end{aligned}
\end{equation}
Plugging \cref{eq:prob-interp1,eq:prob-interp2} into \cref{eq:gd-decrease} we get
\begin{equation}\label{eq:gd-decrease-prob}
    \begin{aligned}
        \loss(W,\theta) - \loss(W_{GD},\theta) &\leq \|(W\phi_\theta(X)-Y)V\|^2_F \cdot\frac{\E_{\mathbb P}[\sigma^2]^2}{\E_{\mathbb P}[\sigma^4]} \\
        &= \|(W\phi_\theta(X)-Y)V\|^2_F\cdot \left(1+\frac{\Var_{\mathbb P}(\sigma^2)}{\E_{\mathbb P}[\sigma^2]^2}\right)\inv.
    \end{aligned}
\end{equation}
As discussed in \cref{sec_app:modified_loss}, the minimizer $W_{CF} = W^\star(\theta)$ is generally non-unique for $\beta = 0$. By definition of the minimizer, any choice will give the same value $\loss(W_{CF},\theta)$, so without loss of generality we may assume $W_{CF} = Y\phi_\theta(X)^+$, where $\phi_\theta(X)^+\in \R^{n\times d}$ is the Moore-Penrose pseudo-inverse of $\phi_\theta(X)$.
\begin{equation}
    \loss(W_{CF},\theta) = \|Y\phi_\theta(X)^+\phi_\theta(X) -Y\|^2_F \\
    = \|Y(I-VV^\top)\|^2_F \\
    = \|(W\phi_\theta(X)-Y)(I-VV^\top)\|^2_F.
\end{equation}
Then since $VV^\top$ is an orthogonal projection when acting on the right,
\begin{equation}\label{eq:cf-decrease}
    \begin{aligned}
        \loss(W,\theta) &= \|W\phi_\theta(X) - Y\|^2_F \\
        &= \|(W\phi_\theta(X) - Y)VV^\top + (W\phi_\theta(X) - Y)(I-VV^\top)\|^2_F \\
        &= \|(W\phi_\theta(X) - Y)VV^\top\|^2_F + \|(W\phi_\theta(X) - Y)(I-VV^\top)\|^2_F \\
        &= \|(W\phi_\theta(X) - Y)V\|^2_F + \loss(W_{CF},\theta).
    \end{aligned}
\end{equation}
\Cref{thm:one-step} then follows by combining \cref{eq:cf-decrease} with \cref{eq:gd-decrease-prob}.

\section{Alternative Algorithm}\label{sec_app:alternative_algo}

\cref{alg:method} presented in Section~\ref{sec:method} structures the updates in such a way that the backbone parameters are updated on the current batch and the last layer is updated on the future batch. This allows us to keep the last layer parameters up-to date. Moreover, this is theoretically justified since it directly follows the structure of~\eqref{eq:new-sgd-update} and \Cref{thm:stochastic-gradient-equiv}.

Here, we present an alternative approach which we describe in~\cref{alg:method_simple}. In this algorithm, we update the backbone and the last layer on the same batch, but in a reversed order. Overall, we found this approach to lead to a similar performance as~\cref{alg:method}, see~\Cref{sec_app:comparison}. The minor advantage of this method is that it is conceptually easier to implement (since we do not need to track two consecutive batches), but this approach departs from~\eqref{eq:new-sgd-update} and from \Cref{thm:stochastic-gradient-equiv}, and therefore is less theoretically justified.

\begin{algorithm}
\caption{Simple proximal closed-form SGD}
\begin{algorithmic}[1] 
    \State \textbf{Given:} Batch size $B$, proximal coefficient $\lambda > 0$, neural network $\phi_{\theta}$ with initial parameters $\theta_0$, learning rate $\alpha > 0$, initial last layer parameters $W_0$.
    \State $t \gets 0$
    \While{${\theta}_t$ has not converged}
    \State $t \gets t+1$
    \State \textbf{Update backbone on the current batch $\batch_t$}
    \State $\theta_{t} \gets {\theta}_{t-1} - \alpha \nabla_{\theta}\loss_{\batch_t}(  W_{t-1},{\theta}_{t-1})$
    \State \textbf{Update last layer on the current batch $\batch_{t}$}
    \State ${W}_t \leftarrow W^\star_{\mathcal B_{t},{W}_{t-1}}(\theta_t)$
\EndWhile
    \State \textbf{Output:} Optimized $(W^\star, \theta^\star)$
\end{algorithmic}
\label{alg:method_simple}
\end{algorithm}



\section{Deep Feature Instrumental Variable Regression}
\label{sec_app:dfiv}

In Instrumental Variable Regression, we observe a treatment $X$ and an outcome $Y$. But we have an unobserved confounder that affects both $X$ and $Y$, specifically we have the relation
\begin{equation}
    Y = \fst (X)+\epsilon,\quad \E[\epsilon] = 0, \quad \E[\epsilon\mid X] \neq 0
\end{equation}
where $\fst$ is called the structural function which we aim to infer, and $\epsilon$ is an additive noise term. Because $\E[\epsilon\mid X] \neq 0$, we cannot use ordinary supervised learning techniques. Instead we assume we have access to an instrumental variable $Z$ which satisfies $\E[\epsilon\mid Z] = 0$. Then we have that $\E[Y\mid Z] = \E[\fst(X)\mid Z]$, so we solve this equation for $\fst$.

Deep Feature Instrumental Variable Regression (DFIV) \citep{xu_learning_2020} solves this by using two neural networks. The first neural network models $w^\top \psi_{\theta_X}(x) = \fst(x)$, and the second neural network models $W\phi_{\theta_Z}(z) = \E[\psi_{\theta_X}(X)\mid Z=z]$. Training alternates between two stages. In the first stage, $W$ and $\theta_Z$ are fitted so that
\begin{equation}\label{eq:iv1}
    W\phi_{\theta_Z}(z) = \E[\psi_{\theta_X}(X)\mid Z=z],
\end{equation}
and in the second stage, $w$ and $\theta_X$ are fitted so that
\begin{equation}\label{eq:iv2}
    w^\top W \phi_{\theta_Z}(z) = \E[Y \mid Z = z].
\end{equation}
When both~\cref{eq:iv1,eq:iv2} are simultaneously satisfied we see that $\E[w^\top\psi_{\theta_X}(X)\mid Z] = \E[Y\mid Z] = \E[\fst(X)\mid Z]$, as required. Since the first-stage solution $W$ depends on $\theta_X$, this is a bilevel optimization problem, and we alternate between the two stages.

\begin{algorithm}
\caption{DFIV proximal}
\begin{algorithmic}[1] 
    \State \textbf{Given:} Stage $1$ data $\{(x_i^{(1)},z_i^{(1)})\}$, stage $2$ data $\{(y_i^{(2)},z_i^{(2)})\}$, batch sizes $B_1, B_2$, proximal coefficients $\lambda_1,\lambda_2,\lambda_{1,2} > 0$, neural networks $\psi_{\theta_X}, \phi_{\theta_Z}$ with initial parameters $\theta_{X_0}, \theta_{Z_0}$ respectively, learning rates $\alpha_1, \alpha_2 > 0$, initial last layer parameters $w_0, W_0$, number of updates in each stage $T_1,T_2$.
    \State $t_1 \gets 0$
    \State $t_2 \gets 0$
    \While{$\theta_{Z_{t_1}}$ and $\theta_{X_{t_2}}$ have not converged}
    \State Sample $B_1$ stage $1$ data $\batch^{(1)} \subset \{(x_i^{(1)},z_i^{(1)})\}$, and $B_2$ stage $2$ data $\batch^{(2)} \subset \{(y_i^{(2)},z_i^{(2)})\}$
    \For{$t=1$ to $T_1$}
    \State $t_1 \gets t_1 +1$
    \State $\theta_{Z_{t_1}} \gets \theta_{Z_{t_1-1}} - \alpha_1 \nabla_{\theta_Z}  \loss^{(1)}_{\batch^{(1)},W_{t_1-1}}(W_{t_1-1},\theta_{Z_{t_1-1}}, \theta_{X_{t_2}};\lambda_1)$
    \State $W_{t_1} \gets W^\star_{\batch^{(1)}, W_{{t_1-1}}}(\theta_{Z_{t_1}},\theta_{X_{t_2}};\lambda_1)$ using~\eqref{eq:first_stage_batch_prox_last_layer} with proximal coefficient $\lambda_1$
    \EndFor
    \For{$t=1$ to $T_2$}
    \State $t_2 \gets t_2 +1$
    \State Form $W^\star_{\batch^{(1)}, W_{t_1}}(\theta_{Z_{t_1}},\theta_X;\lambda_{1,2})$ as a function of $\theta_X$ via~\eqref{eq:first_stage_batch_prox_last_layer}

    \State $\theta_{X_{t_2}} \gets \theta_{X_{t_2-1}} - \alpha_2 \nabla_{\theta_X} \loss^{(2)}_{\batch^{(2)},w_{t_2-1},W_{t_1}}(w_{t_2-1},\theta_{X_{t_2-1}},\theta_{Z_{t_1}};\lambda_{1,2},\lambda_2)$
    \State $w_{t_2} \gets  w^\star_{\batch^{(2)},w_{t_2-1},W_{t_1}}(\theta_{X_{t_2}},\theta_{Z_{t_1}};\lambda_{1,2},\lambda_2)$ using~\eqref{eq:second_stage_batch_prox_last_layer} with proximal coefficient $\lambda_2$
    \EndFor
\EndWhile
    \State \textbf{Output:} Optimized $(w^\star, W^\star, \theta_X^\star, \theta_Z^\star)$
\end{algorithmic}
\label{alg:dfiv}
\end{algorithm}

Similarly to the methodology introduced in~\cref{sec:stochastic}, we fit both stages with stochastic proximal losses, using batches of data $\batch^{(1)} \subset \{(x_i^{(1)},z_i^{(1)})\}$ for the first stage and $\batch^{(2)} \subset \{(y_i^{(2)},z_i^{(2)})\}$ for the second stage. Given the previous estimate $W_{t}$, the first-stage loss with proximal coefficient $\lambda_1 \geq 0$ is
\begin{equation}
    \loss^{(1)}_{\batch^{(1)},W_{t}}(W,\theta_Z, \theta_X;\lambda_1) := \sum_{(x_i^{(1)},z_i^{(1)}) \in \batch^{(1)}}\|W\phi_{\theta_Z}(z_i^{(1)})-\psi_{\theta_X}(x_i^{(1)})\|^2_2+\lambda_1 \|W-W_{t}\|_F^2,
    \label{eq:first_stage_batch_prox}
\end{equation}
and we denote its minimizer with respect to the last layer by
\begin{equation}
    W^\star_{\batch^{(1)}, W_{t}}(\theta_{Z},\theta_X;\lambda_1) = \argmin_{W} \loss^{(1)}_{\batch^{(1)},W_{t}}(W,\theta_Z,\theta_X;\lambda_1).
    \label{eq:first_stage_batch_prox_last_layer}
\end{equation}
Solving for $W$ in closed form makes it depend on $\theta_X$, and DFIV leverages this dependence in the second stage. We make the proximal coefficient explicit in~\cref{eq:first_stage_batch_prox,eq:first_stage_batch_prox_last_layer} because the second stage uses the solution~\eqref{eq:first_stage_batch_prox_last_layer} with a different coefficient $\lambda_{1,2}$. Given the previous estimate $w_{t}$, the second-stage loss is
\begin{equation}
\begin{aligned}
    \loss^{(2)}_{\batch^{(2)},w_{t},W_t}(w,\theta_X,\theta_Z;\lambda_{1,2},\lambda_2) &:= \sum_{(y_i^{(2)},z_i^{(2)}) \in \batch^{(2)}}\|w^\top W^\star_{\batch^{(1)}, W_{t}}(\theta_{Z},\theta_X;\lambda_{1,2}) \phi_{\theta_Z}(z_i^{(2)})-y_i^{(2)}\|^2_2 \\
    &\quad +\lambda_2 \|w-w_{t}\|^2_2,
\end{aligned}
    \label{eq:second_stage_batch_prox}
\end{equation}
and we denote its minimizer with respect to the last layer by
\begin{equation}
    w^\star_{\batch^{(2)},w_t,W_t}(\theta_X,\theta_Z;\lambda_{1,2},\lambda_2) = \argmin_{w}\loss^{(2)}_{\batch^{(2)},w_{t},W_t}(w,\theta_X,\theta_Z;\lambda_{1,2},\lambda_2).
    \label{eq:second_stage_batch_prox_last_layer}
\end{equation}
For our method ``\emph{$\ell_2$ c.f.~proximal ($\lambda)$}'', we therefore use three distinct proximal hyperparameters: $\lambda_{1}$ for the closed-form solution $W^\star_{\batch^{(1)}, W_{t}}(\theta_{Z},\theta_X;\lambda_1)$ in the first stage~\eqref{eq:first_stage_batch_prox_last_layer}, $\lambda_2$ for the closed-form solution $w^\star_{\batch^{(2)},w_t,W_t}(\theta_X,\theta_Z;\lambda_{1,2},\lambda_2)$ in the second stage~\eqref{eq:second_stage_batch_prox_last_layer}, and a separate $\lambda_{1,2}$ for the closed-form solution $W^\star_{\batch^{(1)}, W_{t}}(\theta_{Z},\theta_X;\lambda_{1,2})$ used inside the second stage~\eqref{eq:second_stage_batch_prox}. In practice we found that it was beneficial to take $\lambda_{1,2}<\lambda_1$ See~\cref{alg:dfiv} for full algorithmic details.

For the baseline ``\emph{$\ell_2$ c.f.~ridge ($\beta)$}'', we essentially replace all the proximal penalties in~\cref{eq:first_stage_batch_prox,eq:second_stage_batch_prox} by ridge penalties, i.e.\ $\beta_1 \|W\|_F^2$ and $\beta_2\|w\|_2^2$.

\section{Experimental Details}
\label{sec_app:experimental_details}

\paragraph{Hardware.} For every experiment we used $A100$ GPU. For all the regression experiments, every hyperparameter run was using 1 $A100$ GPU. For ImageNet experiments, every hyperparameter run used $16$ $A100$ GPUs.

\paragraph{Quantum Chemistry regression}. QM9 dataset~\citep{ramakrishnan2014quantum} consists of approximately 133,000 small organic molecules and involves predicting 12 distinct quantum chemical properties from 435-dimensional Coulomb Matrix features to represent the molecular structures. We utilized a 3-layer MLP backbone (256 hidden units per layer, GELU activation) to regress 12 molecular properties from 435-dimensional Coulomb Matrix features. We train all the methods for $50$ epochs and we report test-set mean squared error (MSE). We use either SGD or Adam optimizers for the backbone. For SGD, we use Nesterov momentum with $0.9$ coefficient. For SGD, we initialize the last layer with zeros, while for Adam, we use LeCun initialization.

Following a standard 80/10/10 split, we performed a hyperparameter sweep over learning rates $\alpha \in \{0.1, 0.01, 0.005, 0.001, 0.0005, 0.0001\}$, $\lambda \in \{0.01, 0.1, 1.0, 10.0, 100.0, 1000.0, 5000.0, 10000.0, 100000.0\}$ and $\beta \in \{0.00001, 0.0001, 0.001, 0.01, 0.1, 1.0, 10.0, 100.0, 1000.0, 5000.0, 10000.0, 100000.0 \}$. We selected the best hyperparameters using validation set MSE at the last $20\%$ of the training iterations.

\paragraph{FNO regression.} We use a Fourier Neural Operator (FNO)~\citep{fno} backbone with 4 FNO blocks (width 128, 16 Fourier modes, ReLU activation) as the feature extractor, with a separate linear last layer of dimension 1×128 (plus bias). Training is performed at a subsampled resolution (stride 32), and we additionally report test MSE at the full resolution. We train all methods for 100 epochs and report test-set mean squared error (MSE). We use either SGD optimizer (wiht momentum $0.9$, Nesterov) or Adam optimizers for the backbone. For SGD, we initialize the last layer with zeros, while for Adam, we use LeCun initialization.

Using a 1448/200/400 train/validation/test split, we performed a hyperparameter sweep over learning rates 
$\alpha \in \{10,5,2,1,0.5,0.2,0.1,0.05,0.01,0.005,0.001,0.0005\}$ for SGD and $\alpha \in \{0.1,0.01,0.005,0.001,0.0005,0.0001\}$ for Adam, $\lambda \in \{0.0001,0.001,0.01,0.1,1.0,10,100,1000,10000 \}$ and $\beta \in \{0.0001,0.001,0.01,0.1,1.0,10,100,1000,10000 \}$. We used batch sizes $8,32,128$. All experiments are averaged over 3 random seeds. We selected the best hyperparameters using validation set MSE at the end of the training.

\paragraph{DFIV regression.} For the experiments, we follow closely~\citep{xu_learning_2020} and we consider a slightly modified version of \texttt{d-spirtes} task~\citep{dsprites17}. This is an image dataset described by five latent parameters $(\texttt{shape}, \texttt{scale}, \texttt{rotation}, \texttt{posX}, \texttt{posY})$. The images are $64\times64 = 4096$ dimensional. In this experiment, the authors fix the \texttt{shape} parameter to \texttt{heart}, i.e., they only used heart-shaped images. The authors generated data for IV regression in which they use each figure as a treatment variable $X$. Hence, the treatment variable is $4096$-dimensional in this experiment. To make the task more challenging, they used $\texttt{posY}$ as the hidden confounder, which is not revealed to the model. They used three latent varaibles as the instrument variables $Z$. The outcome $Y$ is defined as 
\begin{equation}
    Y = f_{\text{struct}}(X) + 32 (\texttt{posY} - 0.5) + \epsilon,
\end{equation}
where $\epsilon \sim \mathcal{N}(0, 0.5)$. Here, we used $f_{\text{struct}}(X)$ from a different paper~\citep{xu_deep_2021}, which was defined as
\begin{equation}
    f_{\text{struct}}(X) = \frac{(\texttt{vec}(B)^\top X)^2 - 3000}{500},
\end{equation}
where $B \in \mathbb{R}^{64\times64}$, $B_{ij}=\frac{|32-j|}{32}$ and $\texttt{vec}(B)$ collapses the matrix $B$ to a vector of dimensionality $4096$. The choice of this structural function was motivated by~\citep{xu_deep_2021}, because the original choice described in~\cite{xu_learning_2020} led to essentially a constant function (in expectation).

For our experiments, we use different batch sizes. The DFIV method~\citep{xu_learning_2020} essentially corresponds to two-stage ``\emph{$\ell_2$ c.f.~ridge ($\beta)$}'' where we have $\beta_1$ and $\beta_2$ parameters for the first and second stage correspondingly. In our proximal method, DFIV proximal, as described in~\cref{sec_app:dfiv}, we have three parameters $\lambda_1$, $\lambda_2$ for the first and second stage proximal updates and $\lambda_{1,2}$ for first-stage update inside the second stage. In practice, we sweep over $\lambda_1$ and $\lambda_2$ and we use $\lambda_{1,2}$ to be very small, i.e. $\lambda_{1,2}=0.0001$ as we found that using large $\lambda_{1,2}$ did not work well. We choose $T_1 = 20$ and $T_2 = 1$ as in~\cite{xu_learning_2020}.

The datasets are split into the training set with $10000$ points, validation set with $100$ points and holdout test set with $488$ points.

When we evaluate performance, we use two strategies. One is following~\citep{xu_learning_2020} and whenever performance is reported, takes the first stage and second stage backbone parameters, and re-estimates the corresponding last layers on the whole $10000$ training set. In Figure~\ref{fig:causality_results_new} it is represented by the solid line. The second strategy just takes the current estimates of the last layers. In Figure~\ref{fig:causality_results_new} it is represented by the dashed line.

As optimizer, we use AdamW with weight decay parameter $0.1$. We use LeCun initialization for the last layer.

The sweep range for $\beta_1,\beta_2,\lambda_1,\lambda_2$ is $\{0.00001,0.0001,0.001,0.01,0.1,1.0,10.0,100.0,10^3,10^4\}$. On top of that, we also sweep over the learning rate (we use the same learning rate for both stages) in the range $\{0.001, 0.005, 0.01, 0.05, 0.1\}$. Each experiment is run with $3$ seeds. The best hyperparameters are selected by minimizing the mean squred error (MSE) on the validation set at the end of the training, using the first evaluation strategy (re-estimating the last layers on the whole dataset). The performance is reported on the holdout test set.

\paragraph{CIFAR-100.} We always use SGD optimizer with Nesterov momentum $\gamma=0.9$.  For SGD, we initialize the last layer with zeros, while for Adam, we use LeCun initialization. We train on the $80\%$ of the training set and we use the remaining $20\%$ for validation. For reporting performance, we use the corresponding test set. The sweep ranges are $\alpha \in \{10.0, 5.0, 2.0, 1.0, 0.5, 0.3, 0.2, 0.1, 0.05, 0.01, 0.005, 0.001, 0.0005\}$, learning rate and $\lambda \in \{0.0001, 0.001, 0.01, 0.1, 1.0, 10.0, 100.0, 1000.0]$, $\beta \in \{0.0001, 0.001, 0.01, 0.1, 1.0, 10.0, 100.0, 1000.0\}$.

\paragraph{ImageNet.} We follow closely the experimental setup described in~\citep{brock2021highperformancelargescaleimagerecognition}, including learning rate schedule, label smoothing, data augmentations and Nesterov momentum in the SGD. The learning schedule is a warmup cosine decay with the peak learning rate $\alpha=1.6$. We also swept over $\alpha \in [0.1,1.,1.6,2.,5.0]$ range. We use zeros initialization for the last layer. Overall, all the methods performed the best with $\alpha=1.6$ except for ``\emph{$\ell_{2}$ loss}'' which performed the best with $\alpha=1$. For ``\emph{$\ell_2$ c.f. proximal ($\lambda$)}'', we used $\lambda=10000$ and for ``\emph{$\ell_2$ c.f. ridge ($\beta$)}'', we used $\beta=0.01$. To select these parameters, we ran a sweep over $\beta \in [0.0001, 0.001, 0.01, 0.1, 1.0, 10.0, 100.0, 1000.0, 10000.0]$ and over $\lambda \in [0.0001, 0.001, 0.01, 0.1, 1.0, 10.0, 100.0, 1000.0, 10000.0]$. We used validation set of ImageNet for the hyperparmaeters selection.




\section{Additional Results}
\label{sec_app:additional_results}

\subsection{Application to Classification}
\label{sec_app:classification_results}


\paragraph{CIFAR-100.} We perform experiments on the CIFAR-100 dataset~\citep{cifar10}, across batch sizes $B=[32,128,1024,4096]$, where we use ResNet-18~\citep{resnet} as a backbone $\phi_{\theta}$. Please refer to~\Cref{sec_app:experimental_details} for more details.

The results are presented in Figure~\ref{fig:newest_results}. Consistent with the regression setting, our method ``\emph{$\ell_2$ c.f.~proximal ($\lambda)$}'' outperforms the ``\emph{$\ell_2$ loss}'' baseline under both optimizers, with the performance gap widening as batch size increases. The ridge method, ``\emph{$\ell_2$ c.f.~ridge ($\beta)$}'' fails to converge at small batch sizes and only reaches comparable performance at larger ones -- underscoring the role of the proximal term in preventing overfitting to individual mini-batches. Throughout, our method remains competitive and robust regardless of the choice between SGD and Adam. Notably, ``\emph{$\ell_2$ c.f.~proximal ($\lambda)$}'' even surpasses \emph{Cross Entropy} on CIFAR-100, though as we show below, this advantage does not persist at the larger scale of ImageNet.

\begin{figure}[h]
\begin{center}
\includegraphics[width=5in]{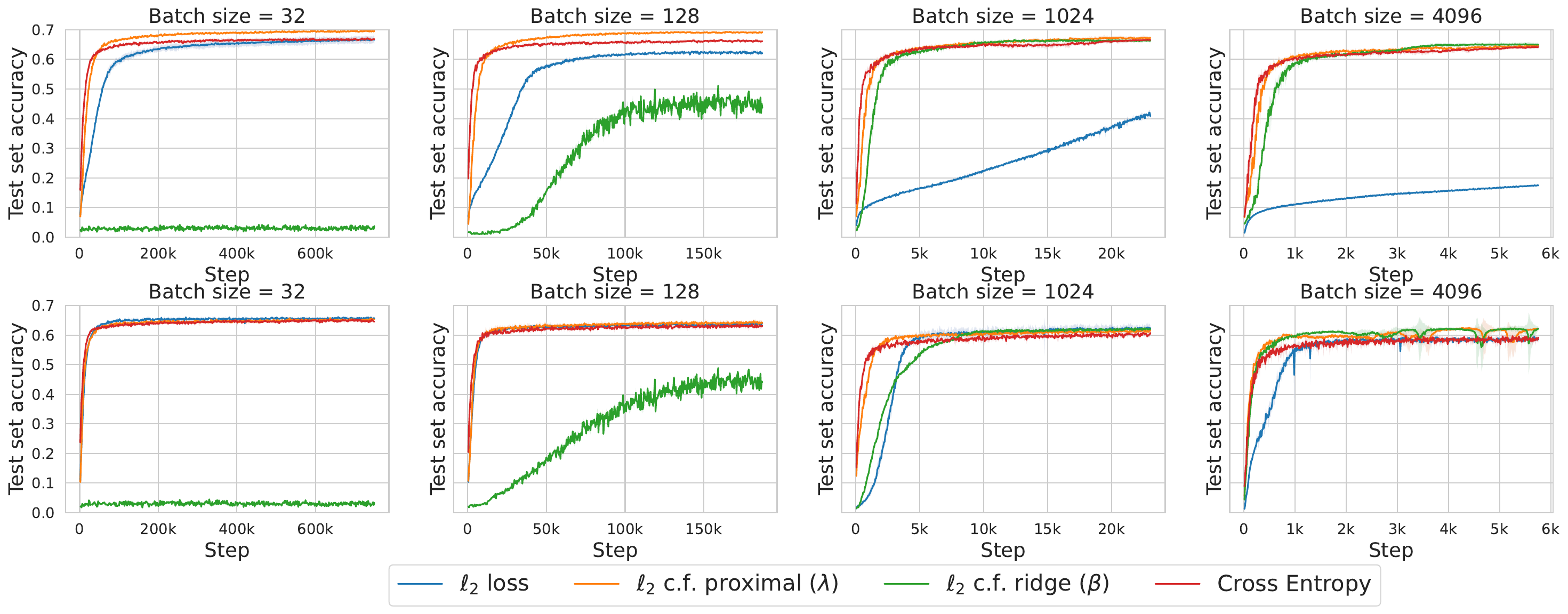}
\end{center}
\caption{\textbf{CIFAR-100 results.}. X-axis is the number of iterations, Y-axis is a test set accuracy. Each column corresponds to a different batch size. Different colors indicate different methods.}
\label{fig:newest_results}
\end{figure}



\paragraph{Impact of $\lambda$ and $\beta$.} In Figure~\ref{fig:cifar100_new_hparams}, we report performance at the end of the training as a function $\lambda$ and $\beta$, as well as the best learning rate $\alpha$ for every batch size. For the first two plots we used the learning rates reported in the third plot. The method ``\emph{$\ell_2$ c.f.~proximal ($\lambda)$}''  is overall robust to $\lambda$ provided it is large enough. We only see some sensitivity for smaller batch sizes. The approach ``\emph{$\ell_2$ c.f.~ridge ($\beta)$}'' is more sensitive to the parameter $\beta$ and works better for larger batch sizes. Finally, both of the approaches benefit from large learning rates whenever batch size is increased, while ``\emph{Cross Entropy}'' and ``\emph{$\ell_2$ loss}'' require small learning rates.


\begin{figure}[h]
\centering
\includegraphics[width=5.3in]{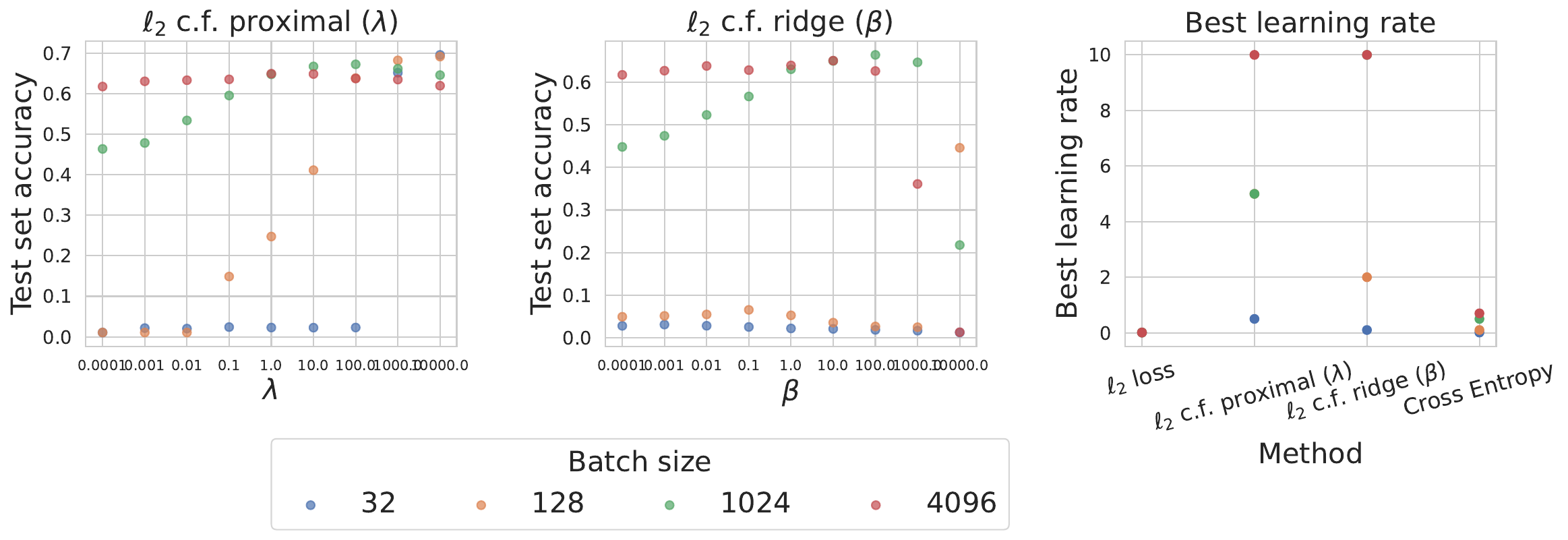} \\
\vspace{0.2in} 
\includegraphics[width=5.3in]{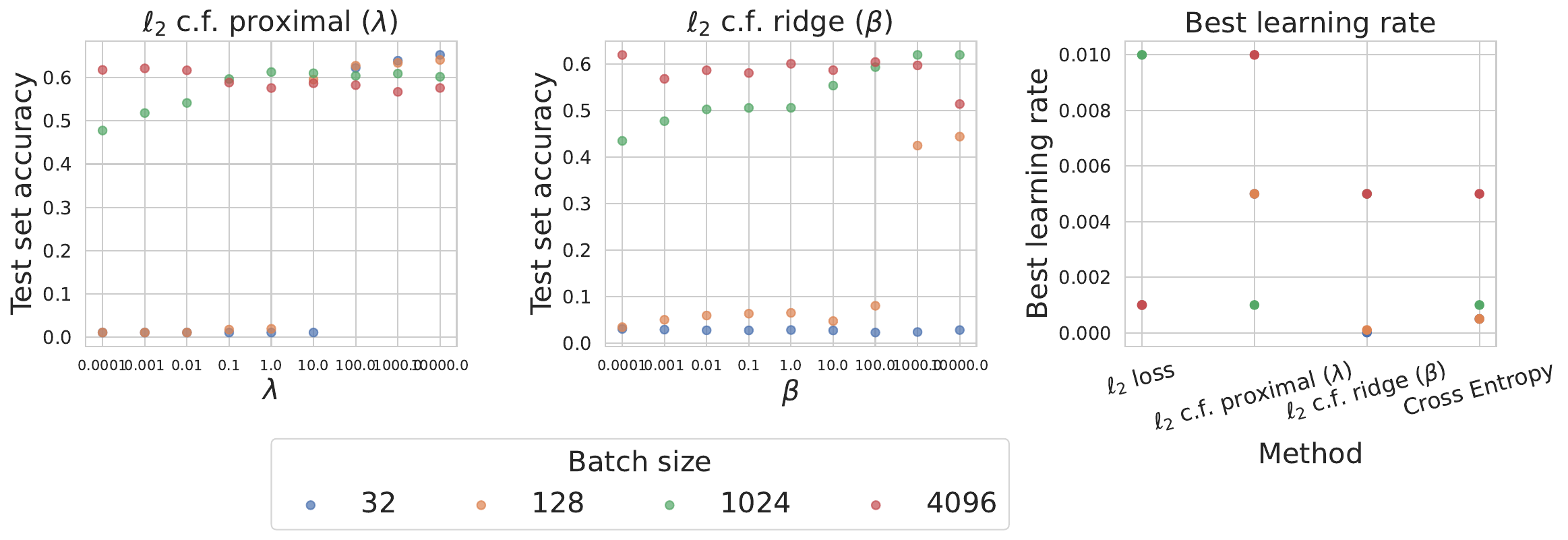}
\caption{\textbf{Dependence on hyperparameters on CIFAR-100}. Top: SGD performance. Bottom: Adam performance. X-axis is the number of iterations, Y-axis is a test set accuracy. \textbf{Left}, ablation over $\lambda$ for `\emph{$\ell_2$ c.f.~proximal ($\lambda)$}''. \textbf{Center}, ablation over $\beta$ for ``\emph{$\ell_2$ c.f.~ridge ($\beta)$}''. \textbf{Right}, the best learning rate per method.}
\label{fig:cifar100_new_hparams}
\end{figure}

\paragraph{Class confidence.}

We run a simple experiment on CIFAR-100 where we track the maximum probability of a correct class. for $\ell_2$ loss, we renormalize this probability to be positive (since the output of the model does not have to be positive nor normalized to 1), i.e., the probabilities are given by
\begin{equation}
    p_c=\frac{\tilde{p}_c}{\sum_{c'} \tilde{p}_{c'}},
\end{equation}
where 
\begin{equation}
    \tilde{p}_c = f(x)_c - \min_{c'} f(x)_{c'},
\end{equation}
and $f(x)=(f(x)_{1},\ldots,f(x)_{C})$ is the output of the model. We present results in Table~\ref{tab:max_probs}. We see that our method has much lower maximum probabilities compared to cross entropy, which means that the results are less confident. This could be, in principle, problematic for datasets with large numbers of classes. In future work, we will investigate  empirical strategies to improve this performance.

\begin{table}[htb]
\centering
\caption{Comparison of maximum probabilities}
\label{tab:max_probs}
\begin{tabular}{rrr}
\toprule
Batch size & Cross entropy max probs & Our method max probs \\
\midrule
32 & 0.99 & 0.63 \\
128 & 0.99 & 0.62 \\
1024 & 0.99 & 0.73 \\
4096 & 0.99 & 0.76 \\
\bottomrule
\end{tabular}
\end{table}

\paragraph{Additional ablations.} We ran ablations on the design choices for our method. We first verified that the inclusion of a bias term in the last layer did not lead to a difference in performance (see~\Cref{fig:cifar100_new_bias_no_bias}). Further, we found that the \emph{zeros} initialization strategy led to the best results (see~\Cref{fig:cifar100_new_initialization}) when using SGD optimizer, consistent to the QM9 experiments.

\paragraph{Vision Transformer backbone.} We additionally run our method with a Vision Transformer backbone on CIFAR-100. We trained ViT-S with patch size $(4,4)$ (images are $32 \times 32$), following the CIFAR-100 setup above, with an SGD backbone, at two batch sizes, $128$ and $1024$. In~\Cref{tab:vit_cifar100}, we report test accuracy at the end of training, averaged over $3$ seeds. These results are in line with what we report for ResNets on CIFAR-100.

\begin{table}[h]
\centering
\caption{\textbf{ViT-S on CIFAR-100}. Test accuracy (\%) at the end of training, averaged over $3$ seeds, with an SGD backbone.}
\label{tab:vit_cifar100}
\begin{tabular}{lcc}
\toprule
Method & $B = 128$ & $B = 1024$ \\
\midrule
Cross Entropy & $45.0$ & $46.2$ \\
$\ell_2$ c.f.~proximal ($\lambda$) & $47.4$ & $50.4$ \\
\bottomrule
\end{tabular}
\end{table}

\paragraph{Iterative last-layer solves for cross entropy.} There is no closed-form solution for the last layer under the cross-entropy loss. One could use iteratively reweighted least squares (IRLS), which is equivalent to Newton steps on the last layer. We implemented a variant replacing our closed form with $K$-step L-BFGS ($K = 1, 2, 5$) on the cross-entropy loss for fixed backbone features. The setup follows the CIFAR-100 experiments above. 
Results are reported in~\Cref{tab:ce_lbfgs}. We see that using $5$ L-BFGS steps outperforms both ``\emph{Cross Entropy}'' and L-BFGS with $1$ step. This suggests a viable practical strategy to approximate the closed-form solution for losses other than $\ell_2$. Interestingly, this variant is still less effective than our $\ell_2$ closed-form solution on CIFAR-100.

\begin{table}[h]
\centering
\caption{\textbf{Approximating the last layer for cross entropy on CIFAR-100}. The closed-form update is replaced by $K$ steps of L-BFGS on the cross-entropy loss for fixed backbone features.}
\label{tab:ce_lbfgs}
\begin{tabular}{lc}
\toprule
Method & Test accuracy \\
\midrule
Cross Entropy & $0.662$ \\
$\ell_2$ c.f.~proximal ($\lambda$) & $0.686$ \\
Cross Entropy, last layer via $1$-step L-BFGS & $0.667$ \\
Cross Entropy, last layer via $2$-step L-BFGS & $0.678$ \\
Cross Entropy, last layer via $5$-step L-BFGS & $0.679$ \\
\bottomrule
\end{tabular}
\end{table}

\paragraph{Large scale classification on ImageNet.} We study the performance of our approach on ImageNet. We use NF-Nets-F0 architecture~\citep{brock2021highperformancelargescaleimagerecognition} with batch size $4096$ and the same training regime as in~\citep{brock2021highperformancelargescaleimagerecognition}. We used $1$ seed for these experiments. See~\Cref{sec_app:experimental_details} for more details. The results are given in~\Cref{fig:imagenet_results}. We see that our method achieves better performance than ``\emph{$\ell_2$ loss}'' loss but under-performs ``\emph{Cross Entropy}''. While this is contrary to our finding on CIFAR-100, ImageNet has ten times the number of classes, so is a significantly different regime. The under-performance of the methods based on the squared loss could be due to advantageous properties of the cross entropy loss in classification, or simply because the training practices with cross entropy have been greatly perfected over the years.





\begin{figure}[h]
\begin{center}
\includegraphics[width=3.3in]{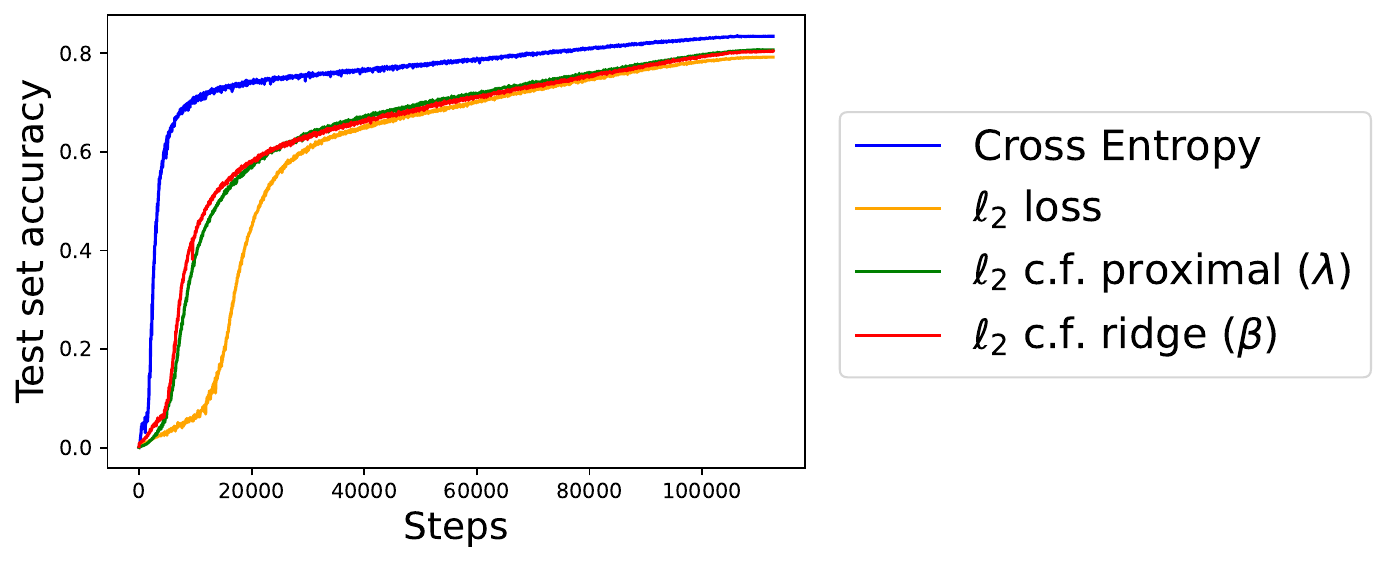}
\end{center}
\vspace{-0.2in}
\caption{\textbf{ImageNet results}. X-axis is the number of iterations, Y-axis is a test set accuracy. Each column corresponds to a different batch size. Different colors indicate different methods.}
\label{fig:imagenet_results}
\end{figure}


\subsection{Impact of Initialization}
\label{sec:initialization_impact}

\paragraph{Impact of initialization on QM9.} We study the impact of initialization of the last layer, discussed in Section~\ref{subsec:numerical_considerations}. The results are given in Figure~\ref{fig:qm9_results}. We see that when we use SGD optimizer, initializing with zeros seems to lead to the best performance (except for very small batch size). When using Adam optimizer, zeros initialization leads to worse results, while all other initializations lead to overall similar performance.

\begin{figure}[h]
\begin{center}
\includegraphics[width=5.3in]{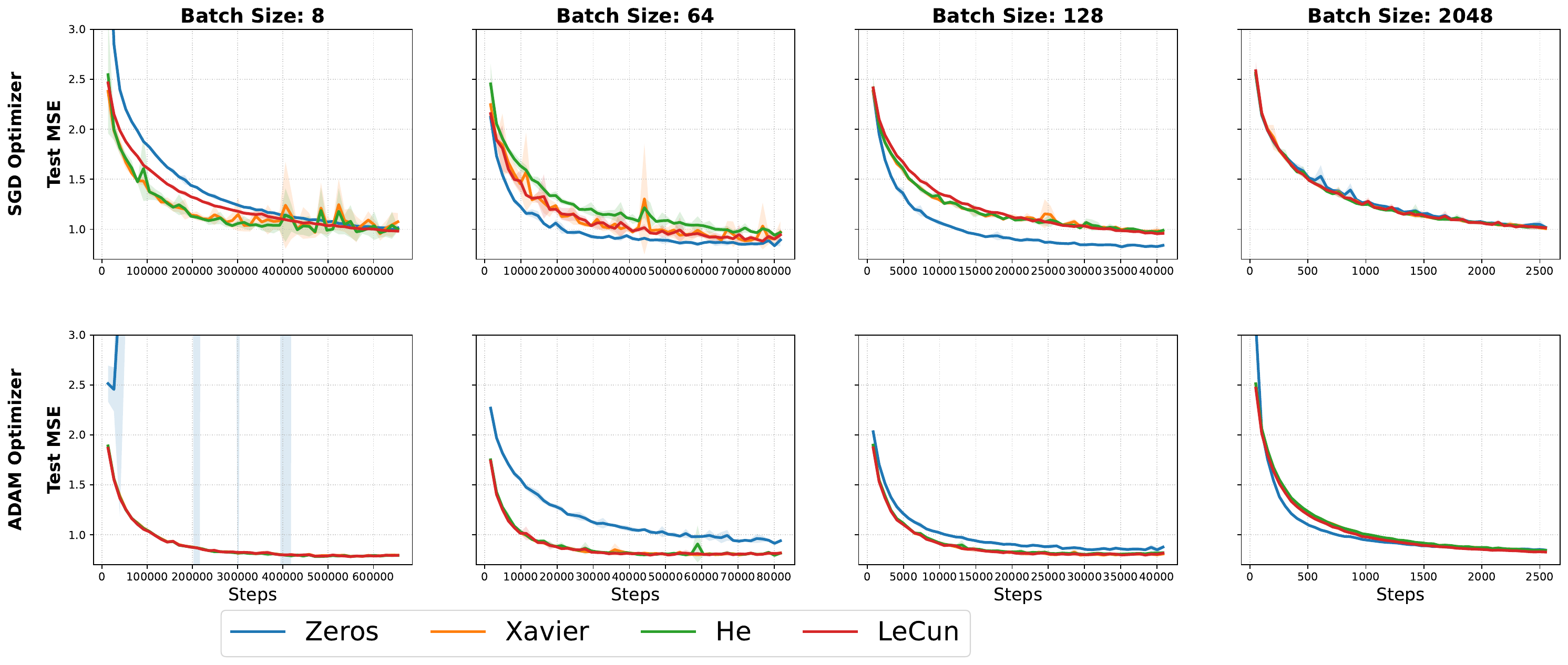}
\end{center}
\caption{\textbf{QM9, impact of initialization}. X-axis is the number of iterations, Y-axis is a test set mean squared error (MSE), columns represents different batch sizes, rows represent different backbone optimizers. Different colors indicate different ways to initialize the last layer.}
\label{fig:qm9_results}
\end{figure}

\paragraph{Impact of initialization on CIFAR-100.} We study impact of different initialization strategies (see~\Cref{subsec:numerical_considerations} for more details) on CIFAR-100 when using SGD optimizer. The results
are given in~\Cref{fig:cifar100_new_initialization}. 
We see that using \emph{zero} initialization leads to overall better performance across batch sizes, which is consistent with our observations on QM9 above.

\begin{figure}[h]
\begin{center}
\includegraphics[width=5in]{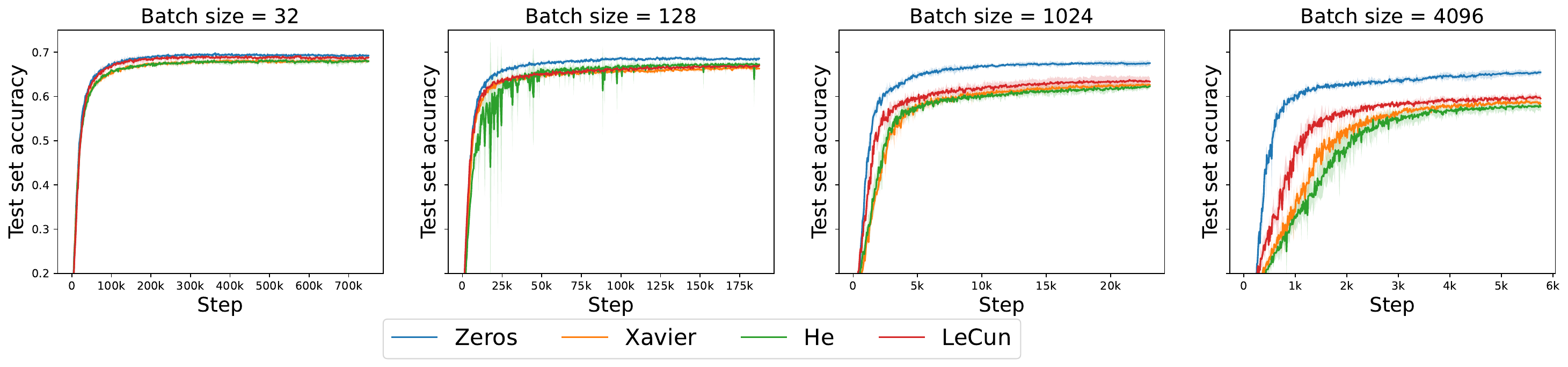}
\end{center}
\caption{\textbf{CIFAR-100, 
impact of initialization}. X-axis is the number of iterations, Y-axis is a test set accuracy. Each column corresponds to a different batch size. Different colors indicate different methods. We use SGD optimizer.}
\label{fig:cifar100_new_initialization}
\end{figure}


\subsection{Use of a Bias Term}
\label{sec:use_of_bias}

\paragraph{Use of a bias on CIFAR-100.} We study the impact of using bias on the performance of ``\emph{$\ell_2$ c.f.~proximal ($\lambda)$}'' on CIFAR-100. The results are 
given in~\Cref{fig:cifar100_new_bias_no_bias}.
We observe similar performance for both strategies.


\begin{figure}[h]
\begin{center}
\includegraphics[width=5in]{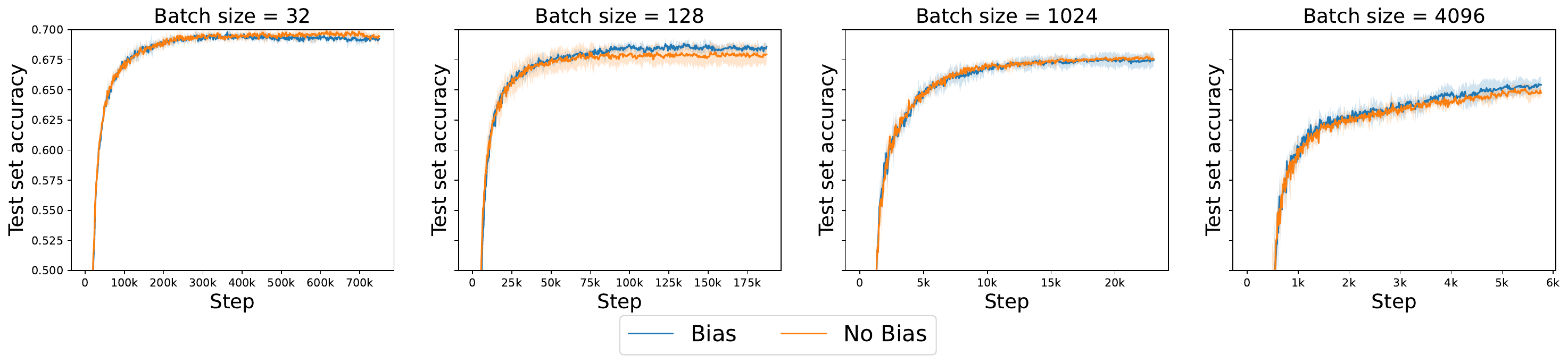}
\end{center}
\caption{\textbf{CIFAR-100, Whether to use a bias}. X-axis is the number of iterations, Y-axis is a test set accuracy. Each column corresponds to a different batch size. Different colors indicate different methods.}
\label{fig:cifar100_new_bias_no_bias}
\end{figure}

\subsection{Comparison of Algorithm~\ref{alg:method} and Algorithm~\ref{alg:method_simple}}
\label{sec_app:comparison}

\paragraph{QM-9.} In Figure~\ref{fig:qm9_results}, we study performance of our method when using Algorithm~\ref{alg:method} or Algorithm~\ref{alg:method_simple}. We see that there is not much difference in performance. The trade-off is that Algorithm~\ref{alg:method_simple} is conceptually easier to implement.

\begin{figure}[h]
\begin{center}
\includegraphics[width=5.3in]{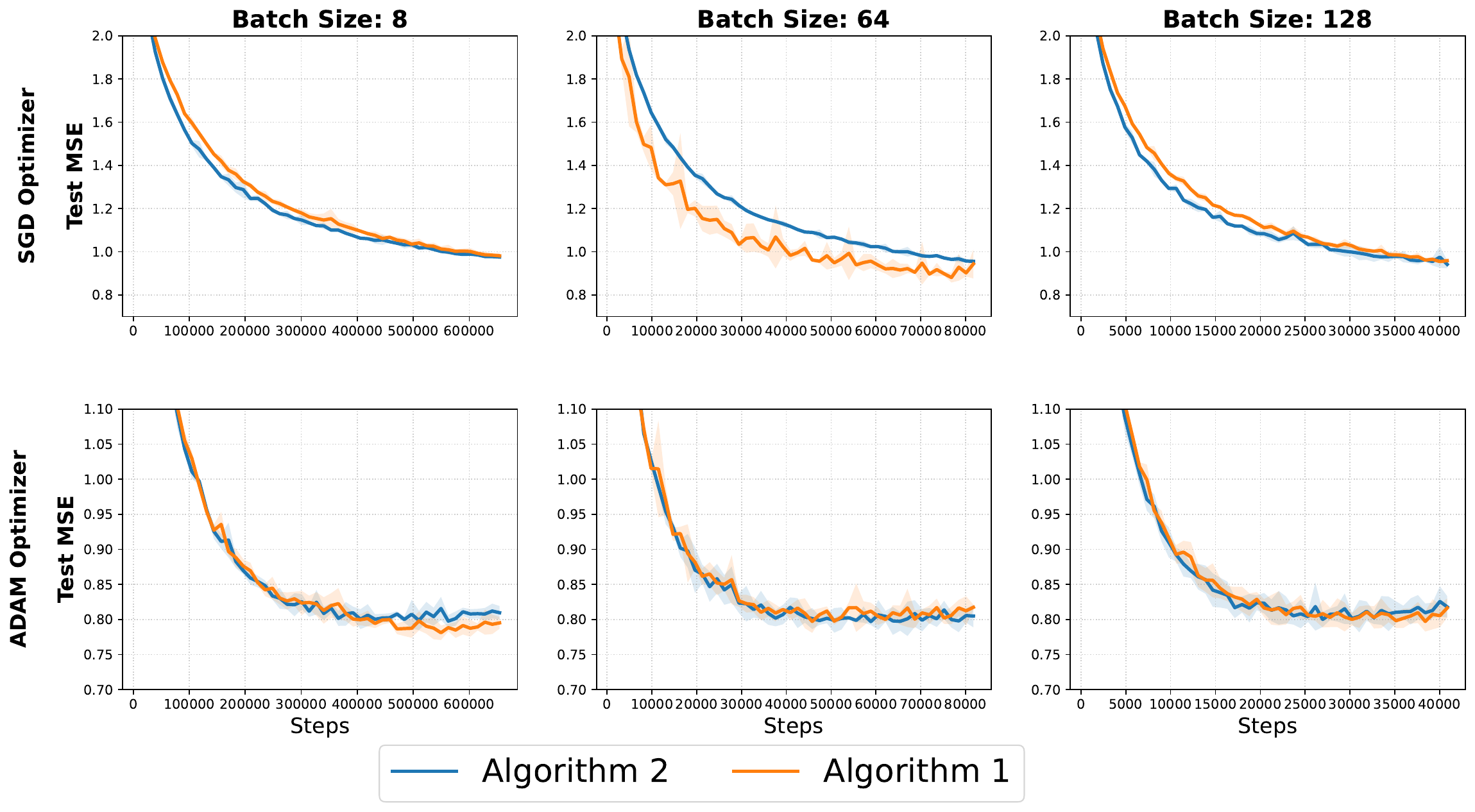}
\end{center}
\caption{\textbf{QM9 algorithms comparison}. X-axis is the number of iterations, Y-axis is a test set mean squared error (MSE), columns represents different batch sizes, rows represent different backbone optimizers. Different colors indicate different algorithms. We use SGD optimizer.}
\label{fig:qm_algorithms}
\end{figure}

\paragraph{CIFAR-100.} We compare the performance of~\cref{alg:method_simple} and~\cref{alg:method} in~\cref{fig:cifar100_order}. We see that there is not much difference in performance. The trade-off is that Algorithm~\ref{alg:method_simple} is conceptually easier to implement.


\begin{figure}[h]
\begin{center}
\includegraphics[width=5.3in]{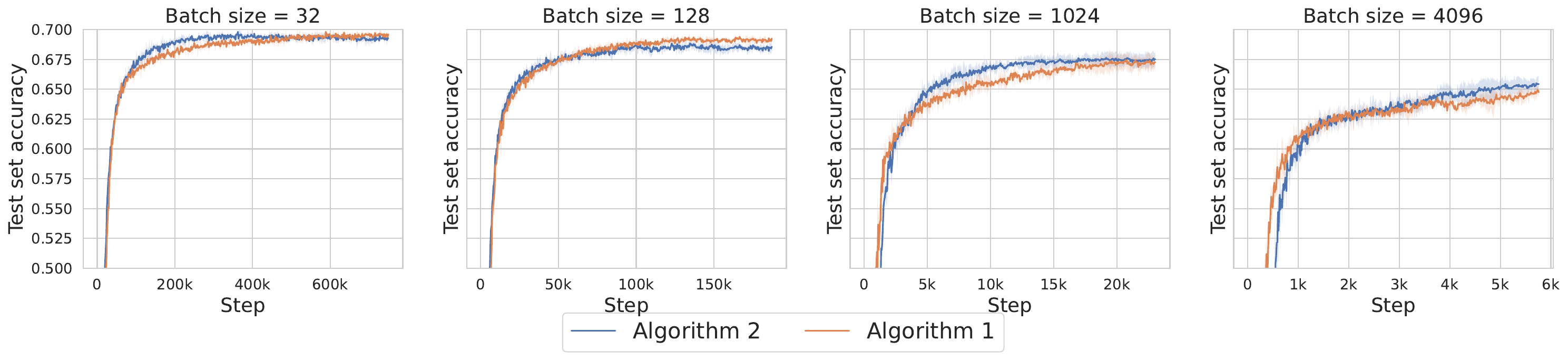}
\end{center}
\caption{\textbf{Comparison of~\cref{alg:method} and~\cref{alg:method_simple}}. X-axis is the number of iterations, Y-axis is a test set accuracy. A column indicates a batch size while a color represents an algorithm.
}
\label{fig:cifar100_order}
\end{figure}

\subsection{Optimization Dynamics}
\label{sec_app:optimization_dynamics}

A natural concern is whether repeatedly recomputing the exact last layer introduces optimization oscillations, or slows backbone feature learning during the early stages of training. We investigate both concerns directly.

\paragraph{Stability across learning rates.} Rather than introducing instability, maintaining the head near its optimum extends the stable learning-rate range: at the largest learning rate in our sweep ($0.1$), plain $\ell_2$ SGD training is unstable, while our method benefits from it (see~\Cref{fig:cifar100_new_hparams}, which reports the best learning rate per method on CIFAR-100).

\paragraph{Backbone dynamics diagnostics.} On QM9 at batch size $128$ we compare ``\emph{$\ell_2$ c.f.~proximal ($\lambda)$}'' ($\lambda = 1000$) to ``\emph{$\ell_2$ loss}'', both with SGD at learning rate $0.005$ (optimal for ``\emph{$\ell_2$ loss}''); our method reaches test MSE $0.836$ vs.\ $1.036$. In~\Cref{tab:grad_norm_ratio}, we report the backbone gradient-norm ratio (ours / ``\emph{$\ell_2$ loss}''), averaged over three stages of training (under matched SGD the update norm is proportional, so it coincides). The backbone thus receives a learning signal of the same magnitude as under standard end-to-end training throughout---including early training---while reaching lower loss. We find no evidence of slowed feature learning.

\begin{table}[h]
\centering
\caption{\textbf{Backbone gradient-norm ratio} on QM9 (``\emph{$\ell_2$ c.f.~proximal ($\lambda)$}'' / ``\emph{$\ell_2$ loss}''), batch size $128$, SGD with learning rate $0.005$, averaged over three windows of training steps.}
\label{tab:grad_norm_ratio}
\begin{tabular}{lccc}
\toprule
Window & 0--1k & 5k--6k & final 1k \\
\midrule
Grad-norm ratio & $0.82$ & $1.03$ & $1.12$ \\
\bottomrule
\end{tabular}
\end{table}

\paragraph{Oscillations.} In the same setting, we measure the cosine similarity between consecutive backbone gradients over the first $2$k steps: $0.08$ for our method vs.\ $0.10$ for SGD. The directional coherence of the backbone gradients is indistinguishable between the two: recomputing the last layer does not measurably perturb backbone optimization.

\subsection{Adaptive Proximal Coefficient}
\label{sec_app:adaptive_lambda}

Our method is not very sensitive to $\lambda$ provided it is sufficiently large (see~\Cref{fig:cifar100_new_hparams} and the study of $\lambda$ on QM9 in~\Cref{sec:experiments}). We additionally implemented an adaptive strategy for $\lambda$ driven by the optimization dynamics---specifically, by how fast the backbone features are changing.

At training step $t$, as the backbone parameters update from $\theta_{t-1}$ to $\theta_t$, we evaluate the minibatch features under both, $\phi_{\theta_{t-1}}(X_t)$ and $\phi_{\theta_t}(X_t)$, and measure the relative feature drift
\begin{equation}
\delta_t = \frac{\|\phi_{\theta_t}(X_t) - \phi_{\theta_{t-1}}(X_t)\|_F}{\|\phi_{\theta_t}(X_t)\|_F + \varepsilon}, \qquad \varepsilon = 10^{-8}.
\end{equation}
We then set $\lambda_t = \lambda_0 / (\delta_t + s)$, with $s > 0$, and use $\lambda_t$ in place of $\lambda$ in the proximal closed-form update. Large drift (early training, large backbone steps) yields a smaller $\lambda_t$, relaxing the anchor to the previous last layer estimate so the last layer can track the changing features; small drift (late training, converged features) yields a larger $\lambda_t \approx \lambda_0 / s$, stabilizing the last layer against minibatch noise. This adds one extra forward pass per step. We used $\lambda_0 = 1000$.

Results on QM9 are reported in~\Cref{tab:adaptive_lambda}. The adaptive rule matches the tuned fixed $\lambda$ under SGD and yields a modest improvement under Adam.

\begin{table}[h]
\centering
\caption{\textbf{Adaptive proximal coefficient on QM9}. Test MSE at batch size $128$, averaged over $3$ random seeds.}
\label{tab:adaptive_lambda}
\begin{tabular}{lccc}
\toprule
 & & \multicolumn{2}{c}{$\ell_2$ c.f.~proximal} \\
\cmidrule(lr){3-4}
Optimizer & $\ell_2$ loss & fixed $\lambda$ & adaptive $\lambda_t$ \\
\midrule
SGD & $1.05$ & $0.84$ & $0.84$ \\
Adam & $0.811$ & $0.818$ & $0.795$ \\
\bottomrule
\end{tabular}
\end{table}

\subsection{Additional Ablations on QM9}
\label{sec_app:qm9_ablations}

\paragraph{L-BFGS backbone optimizer.} In the main paper, we compare under both SGD and Adam backbone optimizers. We additionally ran L-BFGS as the backbone optimizer on QM9, using the standard \texttt{optax} library (disabling line search). We use batch size $128$ and the same learning-rate sweep as in the main paper, and we report test MSE averaged over $3$ seeds in~\Cref{tab:qm9_lbfgs}. Plain ``\emph{$\ell_2$ loss}'' training did not reach acceptable performance with L-BFGS at this batch size across the learning-rate sweep. Notably, our method trains stably and reaches performance comparable to its SGD and Adam results, indicating that the improvements are not tied to a particular backbone optimizer.

\begin{table}[h]
\centering
\caption{\textbf{L-BFGS backbone optimizer on QM9}. Test MSE at batch size $128$, averaged over $3$ seeds.}
\label{tab:qm9_lbfgs}
\begin{tabular}{lc}
\toprule
Method & Test MSE \\
\midrule
$\ell_2$ loss & $12.6$ \\
$\ell_2$ c.f.~proximal ($\lambda$) & $0.8$ \\
\bottomrule
\end{tabular}
\end{table}

\paragraph{MLP width, depth and last-layer dimension.} We run ablations over the MLP size on QM9, varying the width and the depth of the backbone, as well as the dimension of the last hidden layer (i.e., the last-layer feature dimension $d$). We use batch size $128$ and $3$ random seeds, and report test MSE in~\Cref{tab:qm9_mlp_ablations}. Across all configurations our method trains stably, and it improves upon the baseline in every configuration except the two deepest backbones ($[256] \times 6$ and $[256] \times 9$), where both methods perform comparably. The improvement persists at small widths, far from any NTK-like regime, and we observed no optimization pathologies (divergence, oscillations, loss spikes) at any configuration. The corresponding compute and memory costs are analysed in~\Cref{sec_app:wall_clock_synthetnic_qm9}.

\begin{table}[H]
\centering
\caption{\textbf{MLP size ablations on QM9}. Test MSE at batch size $128$, averaged over $3$ seeds. Hidden dimensions list the widths of the hidden layers.}
\label{tab:qm9_mlp_ablations}
\begin{tabular}{llcc}
\toprule
Ablation & Hidden dims & $\ell_2$ loss & $\ell_2$ c.f.~proximal ($\lambda$) \\
\midrule
Width & $[64] \times 3$ & $1.19$ & $1.07$ \\
Width & $[256] \times 3$ & $1.08$ & $0.87$ \\
Width & $[1024] \times 3$ & $1.08$ & $0.82$ \\
Width & $[4096] \times 3$ & $1.43$ & $0.80$ \\
\midrule
Depth & $[256] \times 2$ & $1.74$ & $0.93$ \\
Depth & $[256] \times 3$ & $1.11$ & $0.88$ \\
Depth & $[256] \times 6$ & $0.91$ & $0.93$ \\
Depth & $[256] \times 9$ & $0.92$ & $0.93$ \\
\midrule
Last layer dim & $[256, 256, 64]$ & $1.15$ & $0.88$ \\
Last layer dim & $[256, 256, 256]$ & $1.08$ & $0.83$ \\
Last layer dim & $[256, 256, 1024]$ & $1.62$ & $0.83$ \\
Last layer dim & $[256, 256, 4096]$ & $1.64$ & $0.82$ \\
\bottomrule
\end{tabular}
\end{table}

\section{Wall-Clock Time Analysis on CIFAR-100}
\label{sec_app:wall_clock}

In order to understand the behavior of the proposed method in real scenarios, we conduct a wall-clock time analysis on CIFAR-100. We consider ResNet18 and ResNet50 backbones, where we add an additional last layer of a given dimensionality (\textbf{Last Layer Dim}). We report steps per second (\textbf{SPS}) metrics as well as total time (\textbf{Time}) for training a model for 100 epochs on CIFAR-100 using either our method (Algorithm~\ref{alg:method} or Algorithm~\ref{alg:method_simple}) or cross entropy (CE). We report metrics for different batch sizes and last layer dimensions. Moreover, we report \textbf{rSPS} = SPS(Our method) / SPS(CE) (higher means our method is faster than CE) and \textbf{rTime} = Time(Our method) / Time (CE) (lower means our method is faster than CE). We use A100 GPU. The results for ResNet-18 are given in Table~\ref{tab:resnet18_results} and the results for ResNet-50 are given in Table~\ref{tab:resnet50_results}.

\paragraph{Take-aways.} First, we observe that for small batch size ($32$), \textbf{rTime} of our method increases from $0.92$ (last layer dim = $128$) to $2.37$ (last layer dim = $4096$) for ResNet18, and from $1.66$ (last layer dim = 128) to $1.84$ (last layer dim = $4096$) for ResNet50. This means that for small batch sizes, as we increase the last layer dimension, our method becomes significantly slower than cross entropy. However, we also notice that the \textbf{rTime} decreases as we increase the model size from ResNet18 to ResNet50. This highlights the fact that as the model size increases, the computation required for the last layer becomes relatively smaller compared to the computation of the backbone.

Second, we observe that for large batch size ($1024$), \textbf{rTime} only increases from $1.09$ (last layer dim = $128$) to $1.18$ (last layer dim = $4096$) for ResNet18; and from $1.16$ (last layer dim = $128$) to $1.20$ (last layer dim = $4096$). Finally, for ResNet18, we see that for batch size = $4096$, \textbf{rTime} basically stays very similar (from $1.12$ to $1.14$). This highlights that in the large batch size regime, our method is roughly 10-15\% slower than cross entropy.

\begin{table}[H]
\centering
\caption{Wall clock time comparisons on ResNet-18 with A100 GPU.}
\label{tab:resnet18_results}
\begin{tabular}{cccccccc}
\toprule
Batch Size & Last Layer Dim & SPS (Ours) & SPS (CE) & rSPS & Time (Ours) & Time (CE) & rTime \\
\midrule
32 & 128 & 89.98 & 81.02 & 1.11 & 2164.20 & 2348.11 & 0.92 \\
32 & 256 & 87.35 & 78.97 & 1.11 & 2214.96 & 2400.73 & 0.92 \\
32 & 512 & 78.48 & 79.86 & 0.98 & 2413.57 & 2374.19 & 1.02 \\
32 & 1024 & 70.23 & 81.89 & 0.86 & 2633.73 & 2325.33 & 1.13 \\
32 & 2048 & 51.47 & 79.66 & 0.65 & 3426.96 & 2383.31 & 1.44 \\
32 & 4096 & 30.41 & 82.44 & 0.37 & 5488.86 & 2318.55 & 2.37 \\
128 & 128 & 40.57 & 41.24 & 0.98 & 1293.17 & 1277.97 & 1.01 \\
128 & 256 & 39.59 & 40.74 & 0.97 & 1320.46 & 1281.52 & 1.03 \\
128 & 512 & 38.77 & 40.72 & 0.95 & 1333.75 & 1282.12 & 1.04 \\
128 & 1024 & 35.47 & 40.44 & 0.88 & 1422.86 & 1291.61 & 1.10 \\
128 & 2048 & 29.76 & 41.14 & 0.72 & 1614.35 & 1278.58 & 1.26 \\
128 & 4096 & 20.92 & 40.54 & 0.52 & 2140.42 & 1294.91 & 1.65 \\
1024 & 128 & 4.66 & 5.07 & 0.92 & 1055.30 & 967.64 & 1.09 \\
1024 & 256 & 4.68 & 5.04 & 0.93 & 1046.71 & 975.81 & 1.07 \\
1024 & 512 & 4.68 & 5.09 & 0.92 & 1050.68 & 966.34 & 1.09 \\
1024 & 1024 & 4.66 & 5.05 & 0.92 & 1051.08 & 972.14 & 1.08 \\
1024 & 2048 & 4.57 & 5.10 & 0.89 & 1072.17 & 963.36 & 1.11 \\
1024 & 4096 & 4.21 & 5.02 & 0.84 & 1164.10 & 982.72 & 1.18 \\
4096 & 128 & 1.11 & 1.24 & 0.89 & 1084.92 & 966.99 & 1.12 \\
4096 & 256 & 1.09 & 1.24 & 0.88 & 1097.92 & 968.55 & 1.13 \\
4096 & 512 & 1.08 & 1.24 & 0.87 & 1105.35 & 963.45 & 1.15 \\
4096 & 1024 & 1.10 & 1.23 & 0.90 & 1089.28 & 973.90 & 1.12 \\
4096 & 2048 & 1.09 & 1.23 & 0.89 & 1100.39 & 974.79 & 1.13 \\
4096 & 4096 & 1.08 & 1.24 & 0.88 & 1107.10 & 971.15 & 1.14 \\
\bottomrule
\end{tabular}
\end{table}

\begin{table}[H]
\centering
\caption{Wall clock time comparisons on ResNet-50 with A100 GPU.}
\label{tab:resnet50_results}
\begin{tabular}{cccccccc}
\toprule
Batch Size & Last Layer Dim & SPS (Ours) & SPS (CE) & rSPS & Time (Ours) & Time (CE) & rTime \\
\midrule
32 & 128 & 36.96 & 42.94 & 0.86 & 4735.09 & 2855.43 & 1.66 \\
32 & 256 & 41.18 & 42.84 & 0.96 & 4240.65 & 4056.31 & 1.05 \\
32 & 512 & 35.85 & 39.59 & 0.91 & 4903.87 & 3496.42 & 1.40 \\
32 & 1024 & 35.22 & 43.28 & 0.81 & 4903.89 & 4002.80 & 1.23 \\
32 & 2048 & 30.85 & 41.52 & 0.74 & 5482.36 & 4173.49 & 1.31 \\
32 & 4096 & 21.68 & 41.97 & 0.52 & 7605.07 & 4133.77 & 1.84 \\
128 & 128 & 16.53 & 18.62 & 0.89 & 2645.24 & 2384.37 & 1.11 \\
128 & 256 & 16.37 & 18.49 & 0.89 & 2659.99 & 2397.10 & 1.11 \\
128 & 512 & 16.15 & 18.30 & 0.88 & 2698.11 & 2423.22 & 1.11 \\
128 & 1024 & 15.45 & 18.52 & 0.83 & 2801.06 & 2405.96 & 1.16 \\
128 & 2048 & 14.31 & 18.44 & 0.78 & 3000.61 & 2409.51 & 1.25 \\
128 & 4096 & 11.95 & 18.05 & 0.66 & 3532.58 & 2448.23 & 1.44 \\
1024 & 128 & 2.26 & 2.59 & 0.87 & 2154.44 & 1875.26 & 1.15 \\
1024 & 256 & 2.25 & 2.61 & 0.86 & 2151.89 & 1866.88 & 1.15 \\
1024 & 512 & 2.24 & 2.60 & 0.86 & 2163.26 & 1874.76 & 1.15 \\
1024 & 1024 & 2.23 & 2.59 & 0.86 & 2174.04 & 1876.25 & 1.16 \\
1024 & 2048 & 2.19 & 2.57 & 0.85 & 2216.16 & 1889.88 & 1.17 \\
1024 & 4096 & 2.12 & 2.56 & 0.83 & 2286.04 & 1898.44 & 1.20 \\
\bottomrule
\end{tabular}
\end{table}

\section{Wall-Clock Time, Memory Footprint, and Scalability Analysis on QM9}
\label{sec_app:wall_clock_synthetnic_qm9}

To complement the CIFAR-100 timings in \cref{sec_app:wall_clock}, we evaluate how the wall-clock step time and peak GPU memory of closed-form last-layer optimization scale on QM9~\citep{ramakrishnan2014quantum}. We vary five dimensions of the network $f(x; W, \theta) = W\phi_\theta(x)$: backbone depth $L$, backbone hidden width $H$, last-layer feature dimension $d$ ($\phi_\theta(x) \in \R^d$), batch size $B = |\batch_t|$, and output dimension $o$ ($W \in \R^{o \times d}$). Following \cref{sec_app:wall_clock}, we report step time (ms), \textbf{rTime} $= \text{Time}(\text{Ours}) / \text{Time}(\ell_2\text{ loss})$, relative overhead $(\textbf{rTime}-1)\times 100\%$, and peak GPU memory (GiB).

\subsection{Computational Cost: \cref{alg:method} vs.\ \cref{alg:method_simple}}
\label{subsec_app:complexity_analysis}


On a batch $\batch_t$ of size $B$, let $\Phi_t := \phi_\theta(X_t) \in \R^{d \times B}$ denote the penultimate features. A standard gradient step on ``\emph{$\ell_2$ loss}'' performs one forward pass ($T_{\text{fwd}}$) and one backward pass ($T_{\text{bwd}} \approx 2\,T_{\text{fwd}}$, since each layer computes two matrix multiplications: one for weight gradients $\nabla_{W_\ell}\loss$ and one for activation gradients $\nabla_{h_{\ell-1}}\loss$), totaling $T_{\text{backbone}} = T_{\text{fwd}} + T_{\text{bwd}} \approx 3\,T_{\text{fwd}} = \mathcal{O}(LBH^2 + BHd + Bdo)$ FLOPs. Solving the closed-form normal equations~\eqref{eq:closed-form-proximal} via Cholesky factorization costs

\begin{equation}\label{eq:chol_flops}
    T_{\text{chol}}(B, d, o) = \underbrace{B d^2}_{\text{Gram } \Phi_t\Phi_t^\top} + \underbrace{\tfrac{1}{3} d^3}_{\text{Cholesky of } \Phi_t\Phi_t^\top + \lambda I_d} + \underbrace{2 B d o}_{\text{RHS } Y_t\Phi_t^\top + \lambda W_t} + \underbrace{2 d^2 o}_{\text{Triangular solve}}.
\end{equation}
This decomposition explains four practical scaling properties:
\begin{enumerate}[wide, labelwidth=!, labelindent=0pt]
    \item \textbf{\cref{alg:method} (1 forward pass) vs.\ \cref{alg:method_simple} (2 forward passes).}
    In \cref{alg:method_simple}, $\theta_t$ is updated before evaluating $W_t = W^\star_{\batch_t, W_{t-1}}(\theta_t)$, which requires a second forward pass $\phi_{\theta_t}(X_t)$ on $\batch_t$. Since $T_{\text{backbone}} \approx 3\,T_{\text{fwd}}$, this extra forward pass adds $+1\,T_{\text{fwd}}$, creating an asymptotic $+33.3\%$ overhead floor ($\textbf{rTime} \to 4/3$) as $L \to \infty$. In contrast, \cref{alg:method} evaluates both updates at the same backbone parameters (or reuses the cached features $\Phi_t$ from the forward pass), requiring only \textbf{one forward pass per batch}. Its relative overhead therefore decays to zero with depth $L$:
    \begin{equation}\label{eq:one_pass_amortization}
        \textbf{rTime}_{\text{\cref{alg:method}}}(L, H, d) - 1 = \frac{T_{\text{chol}}(B, d, o)}{3\,T_{\text{fwd}}(L, H, B)} = \mathcal{O}\!\left(\frac{B d^2 + d^3}{L B H^2}\right) \xrightarrow{L \to \infty} 0\%.
    \end{equation}

    \item \textbf{Bottleneck width decoupling ($H \gg d$).}
    Because $T_{\text{chol}}(B, d, o)$ depends only on the last-layer dimension $d$ and not on the hidden width $H$, projecting a wide backbone ($H \in [2048, 8192]$) to a bottleneck dimension $d = 256$ keeps the Cholesky solve under $0.25\text{ ms}$ on an A100 GPU.

    \item \textbf{Invariance to output dimension $o$.}
    Neither the Gram matrix $\Phi_t\Phi_t^\top + \lambda I_d \in \R^{d \times d}$ nor its $\mathcal{O}(d^3)$ Cholesky factorization depends on $o$. Only the right-hand side ($\mathcal{O}(Bdo)$) and triangular back-substitution ($\mathcal{O}(d^2 o)$) depend on $o$.

    \item \textbf{Negligible memory overhead.}
    By \Cref{thm:stochastic-gradient-equiv}, autodiff does not backpropagate through $(\Phi_t\Phi_t^\top + \lambda I_d)\inv$. The only extra memory is the $d \times d$ Gram matrix ($256\text{ KB}$ at $d=256$, and $64\text{ MB}$ at $d=4096$).
\end{enumerate}

\begin{figure}[t]
\centering
\includegraphics[width=\textwidth]{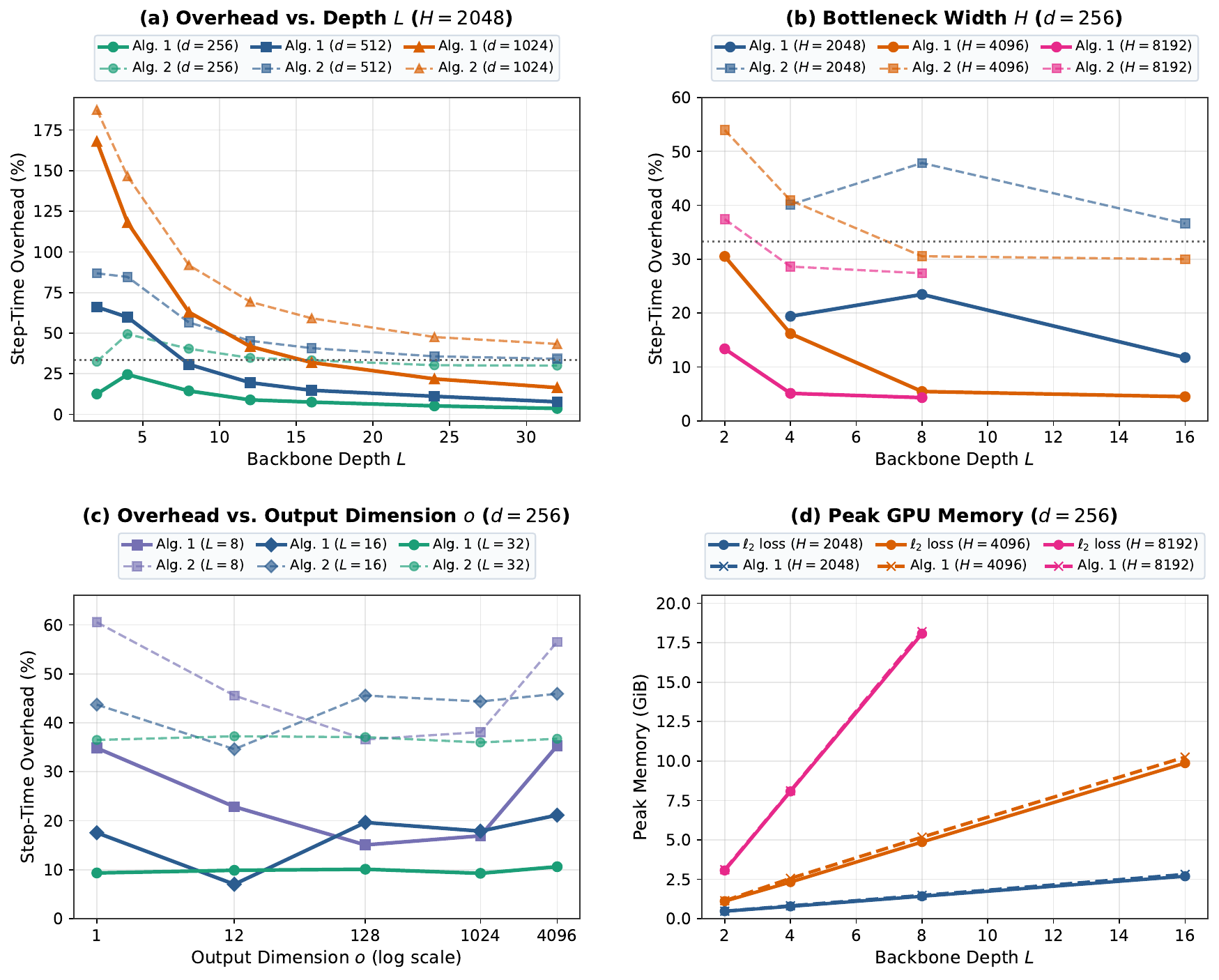}
\caption{\textbf{Scalability summary on QM9 (A100 GPU).} \textbf{(a)}~Relative step-time overhead vs.\ depth $L$ ($H=2048, B=2048$): \cref{alg:method_simple} plateaus at $+33.3\%$, whereas \cref{alg:method} decays as $\mathcal{O}(1/L)$ to \textbf{$+3.6\%$} at $L=32, d=256$. \textbf{(b)}~Wide backbones ($H \in [2048, 8192]$) with a $d=256$ bottleneck achieve \textbf{$<5\%$ overhead} with \cref{alg:method}. \textbf{(c)}~Scaling the output dimension $o$ by $4096\times$ ($o=1 \to 4096$) changes the overhead of \cref{alg:method} by only $+1.3\%$ at $L=32$. \textbf{(d)}~Peak GPU memory of \cref{alg:method} is nearly identical to ``\emph{$\ell_2$ loss}'' ($<1.5\%$ overhead).}
\label{fig:qm9_scalability_summary}
\end{figure}

\subsection{Empirical Scaling Results on A100 GPUs}
\label{subsec_app:deep_amortization_results}

\Cref{fig:qm9_scalability_summary,fig:qm9_depth_bottleneck_output_scaling} and \Cref{tab:deep_amortization_bottleneck} summarize the benchmark results across depth $L$, bottleneck width $H$, and output dimension $o$:
\begin{itemize}[leftmargin=*]
    \item \textbf{Depth $L$ (\cref{fig:qm9_scalability_summary}a, \cref{fig:qm9_depth_bottleneck_output_scaling}a)}: At $d=256$ ($H=2048, B=2048$), the overhead of \cref{alg:method} drops from $+24.5\%$ at $L=4$ to $+14.5\%$ at $L=8$, $+7.6\%$ at $L=16$, and \textbf{$+3.6\%$ ($\textbf{rTime}=1.04$)} at $L=32$. At $d=512$, overhead drops to \textbf{$+7.7\%$} at $L=32$. Even at $d=4096$, increasing depth from $L=1$ to $L=32$ reduces relative overhead by $11\times$ ($+1078\% \to +98.5\%$).
    \item \textbf{Wide backbones with $d=256$ bottleneck (\cref{fig:qm9_scalability_summary}b, \cref{fig:qm9_depth_bottleneck_output_scaling}b)}: At $H=4096, L=16$ ($254.6\text{M}$ parameters) and $H=8192, L=8$ ($475.5\text{M}$ parameters), \cref{alg:method} adds only \textbf{$+4.5\%$} and \textbf{$+4.3\%$} step-time overhead, respectively, while peak GPU memory increases by just $+0.7\%$ ($18.08 \to 18.20\text{ GiB}$).
    \item \textbf{Output dimension $o$ (\cref{fig:qm9_scalability_summary}c, \cref{fig:qm9_depth_bottleneck_output_scaling}c)}: Increasing $o$ from $1$ to $4096$ at $L=32$ ($H=1024, d=256$) changes the overhead of \cref{alg:method} from $+9.3\%$ to $+10.6\%$.
\end{itemize}

\begin{figure}[t]
\centering
\includegraphics[width=\textwidth]{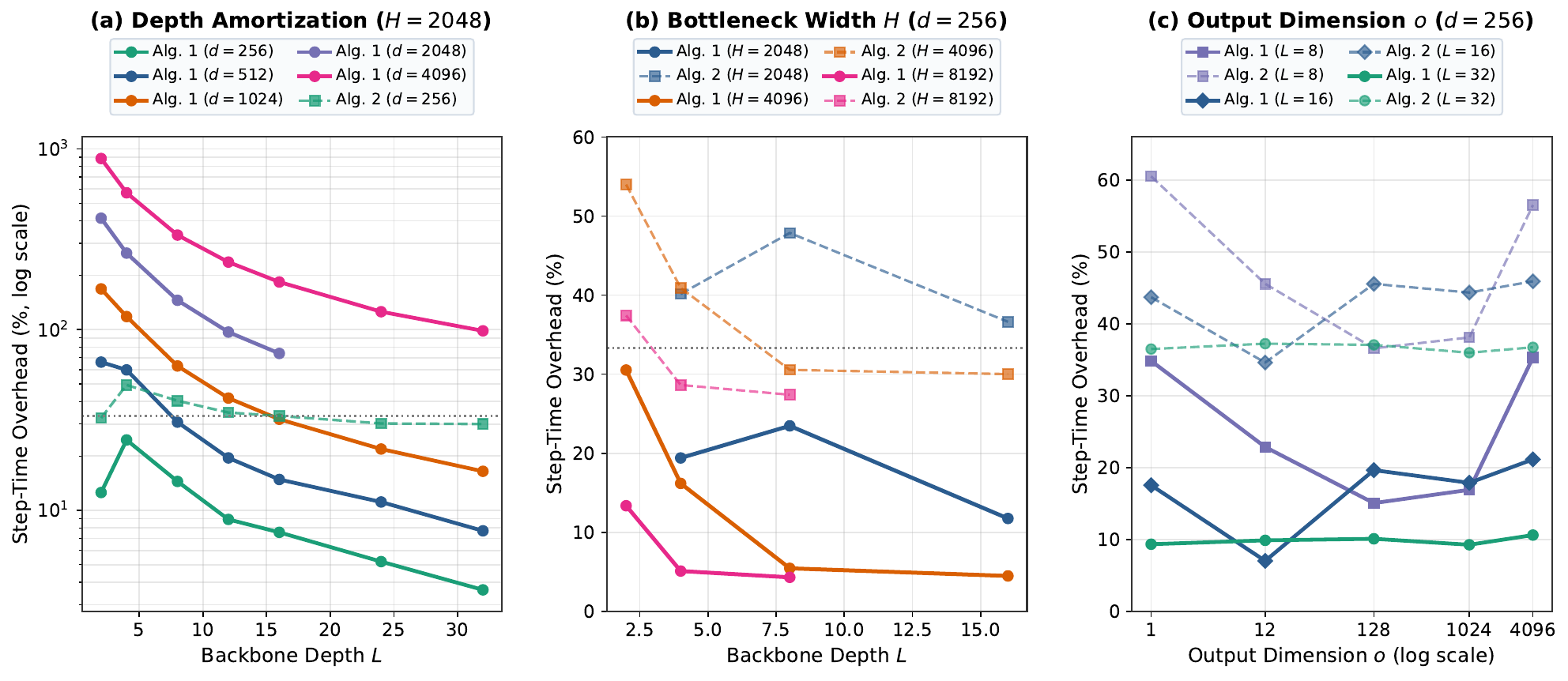}
\caption{\textbf{Detailed scaling curves on A100 GPU.} \textbf{(a)}~Log-scale overhead vs.\ depth $L$ across $d \in \{256, 512, 1024, 2048, 4096\}$. \textbf{(b)}~Bottleneck regime ($d=256$) across hidden widths $H \in \{2048, 4096, 8192\}$. \textbf{(c)}~Scaling with output dimension $o \in \{1, 12, 128, 1024, 4096\}$.}
\label{fig:qm9_depth_bottleneck_output_scaling}
\end{figure}

\begin{table}[t]
\centering
\caption{\textbf{Wall-clock step time, \textbf{rTime}, and peak GPU memory on QM9 (A100 GPU).} Comparison of ``\emph{$\ell_2$ loss}'' (SGD), \cref{alg:method_simple}, and \cref{alg:method}.}
\label{tab:deep_amortization_bottleneck}
\footnotesize
\setlength{\tabcolsep}{4pt}
\resizebox{\textwidth}{!}{%
\begin{tabular}{lccccccccc}
\toprule
\textbf{Regime} & $H$ & $d$ & $L$ & \textbf{Params} & \textbf{$\ell_2$ loss} & \textbf{\cref{alg:method_simple} (rTime)} & \textbf{\cref{alg:method} (rTime)} & \textbf{Overhead (\cref{alg:method})} & \textbf{Memory (GiB)} \\
\midrule
\multirow{8}{*}{\shortstack[l]{Depth $L$\\($B=2048$,\\$o=12$)}} 
 & 2048 & 256 & 4  & 14.0M  & 2.72 ms  & 4.07 ms (1.49$\times$) & \textbf{3.39 ms (1.25$\times$)} & \textbf{$+24.5\%$} & 0.81 $\to$ 0.87 \\
 & 2048 & 256 & 8  & 30.8M  & 5.29 ms  & 7.43 ms (1.40$\times$) & \textbf{6.06 ms (1.15$\times$)} & \textbf{$+14.5\%$} & 1.44 $\to$ 1.50 \\
 & 2048 & 256 & 16 & 64.4M  & 10.66 ms & 14.20 ms (1.33$\times$) & \textbf{11.46 ms (1.08$\times$)} & \textbf{$+7.6\%$} & 2.69 $\to$ 2.81 \\
 & 2048 & 256 & 32 & 131.5M & 21.35 ms & 27.76 ms (1.30$\times$) & \textbf{22.12 ms (1.04$\times$)} & \textbf{$+3.6\%$} & 5.33 $\to$ 5.46 \\
 \cmidrule(lr){2-10}
 & 2048 & 512 & 32 & 132.0M & 21.41 ms & 28.73 ms (1.34$\times$) & \textbf{23.06 ms (1.08$\times$)} & \textbf{$+7.7\%$} & 5.35 $\to$ 5.44 \\
 & 2048 & 1024 & 32 & 133.1M & 21.60 ms & 30.96 ms (1.43$\times$) & \textbf{25.15 ms (1.16$\times$)} & \textbf{$+16.4\%$} & 5.37 $\to$ 5.53 \\
 & 2048 & 2048 & 32 & 135.2M & 21.84 ms & 36.31 ms (1.66$\times$) & \textbf{30.42 ms (1.39$\times$)} & \textbf{$+39.3\%$} & 5.45 $\to$ 5.61 \\
 & 2048 & 4096 & 32 & 139.4M & 22.55 ms & 50.48 ms (2.24$\times$) & \textbf{44.77 ms (1.99$\times$)} & \textbf{$+98.5\%$} & 5.61 $\to$ 6.07 \\
\midrule
\multirow{4}{*}{\shortstack[l]{Bottleneck\\($d=256$,\\$B=1024$)}}
 & 4096 & 256 & 8  & 120.3M & 9.87 ms  & 12.89 ms (1.31$\times$) & \textbf{10.41 ms (1.05$\times$)} & \textbf{$+5.5\%$} & 4.86 $\to$ 5.17 \\
 & 4096 & 256 & 16 & 254.6M & 20.30 ms & 26.39 ms (1.30$\times$) & \textbf{21.21 ms (1.04$\times$)} & \textbf{$+4.5\%$} & 9.87 $\to$ 10.25 \\
 & 8192 & 256 & 4  & 207.0M & 14.86 ms & 19.11 ms (1.29$\times$) & \textbf{15.62 ms (1.05$\times$)} & \textbf{$+5.1\%$} & 8.07 $\to$ 8.11 \\
 & 8192 & 256 & 8  & 475.5M & 33.54 ms & 42.73 ms (1.27$\times$) & \textbf{34.99 ms (1.04$\times$)} & \textbf{$+4.3\%$} & 18.08 $\to$ 18.20 \\
\midrule
\multirow{2}{*}{\shortstack[l]{Output $o$\\($H=1024, L=32$)}}
 & \multicolumn{3}{c}{$o = 1$ ($d=256, B=2048$)}    & 33.8M & 7.66 ms & 10.45 ms (1.36$\times$) & \textbf{8.37 ms (1.09$\times$)}  & \textbf{$+9.3\%$}  & 1.43 $\to$ 1.49 \\
 & \multicolumn{3}{c}{$o = 4096$ ($d=256, B=2048$)} & 34.9M & 8.06 ms & 11.02 ms (1.37$\times$) & \textbf{8.91 ms (1.11$\times$)} & \textbf{$+10.6\%$} & 1.54 $\to$ 1.62 \\
\bottomrule
\end{tabular}%
}
\end{table}

\subsection{Shallow Backbones ($L=3$) and End-to-End QM9 Training}
\label{subsec_app:shallow_and_e2e}

Finally, \Cref{fig:qm9_shallow_heatmaps,fig:qm9_end_to_end_ablations} examine shallow $3$-layer MLPs ($L=3$) across $d \in [64, 4096]$. When $d \leq 512$, \cref{alg:method_simple} and \cref{alg:method} add $<0.5\%$ peak memory and low step-time overhead ($\textbf{rTime} \in [0.97, 1.21]$ at $B=128$) while achieving $20\%$--$50\%$ lower test MSE and reaching the final test MSE of ``\emph{$\ell_2$ loss}'' in $2\times$--$4\times$ less wall-clock time. When $d \geq 2048$ on a shallow network ($L=3$), the $\mathcal{O}(d^3)$ Cholesky factor dominates step time; this regime is avoided by using either a deeper backbone (\cref{fig:qm9_depth_bottleneck_output_scaling}a) or a $d=256$ bottleneck layer (\cref{fig:qm9_depth_bottleneck_output_scaling}b).

\begin{figure}[t]
\centering
\includegraphics[width=0.92\textwidth]{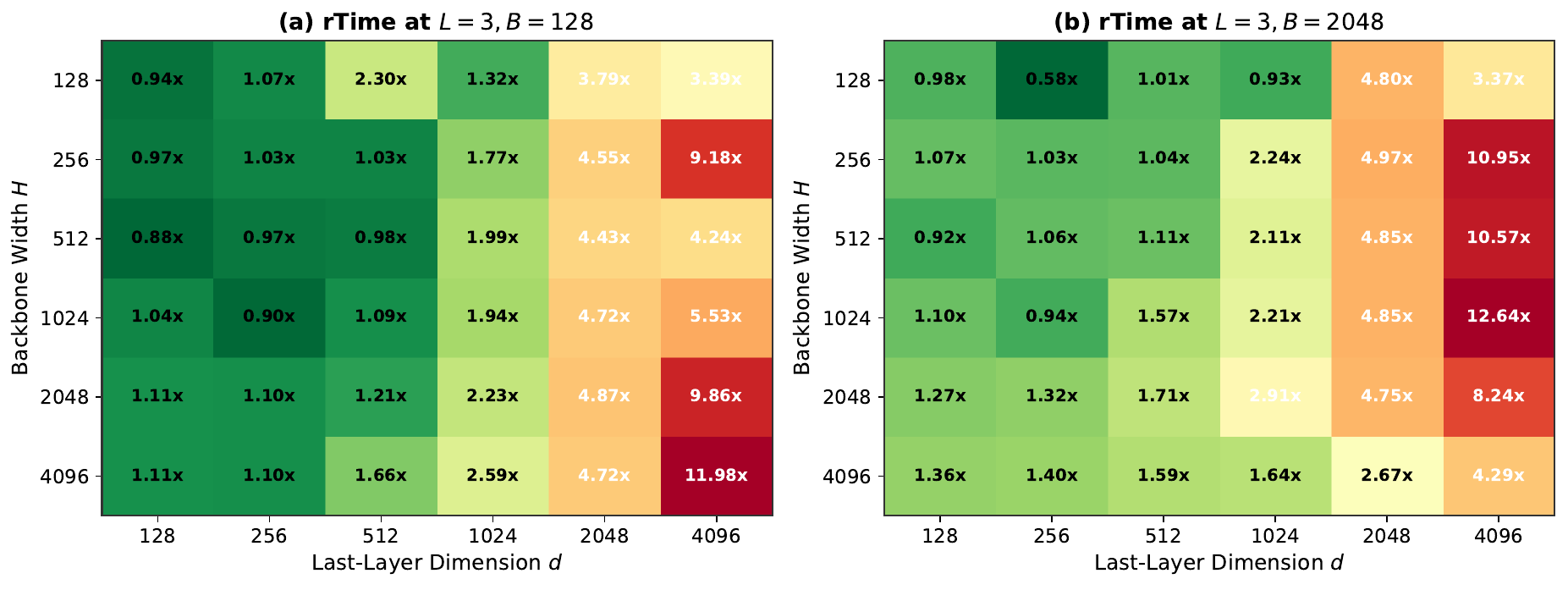}
\caption{\textbf{Shallow-backbone ($L=3$) \textbf{rTime} heatmaps across $(H, d)$ at $B=128$ and $B=2048$.} For $d \leq 512$, $\textbf{rTime}$ remains near $1.0\times$.}
\label{fig:qm9_shallow_heatmaps}
\end{figure}

\begin{figure}[t]
\centering
\includegraphics[width=\textwidth]{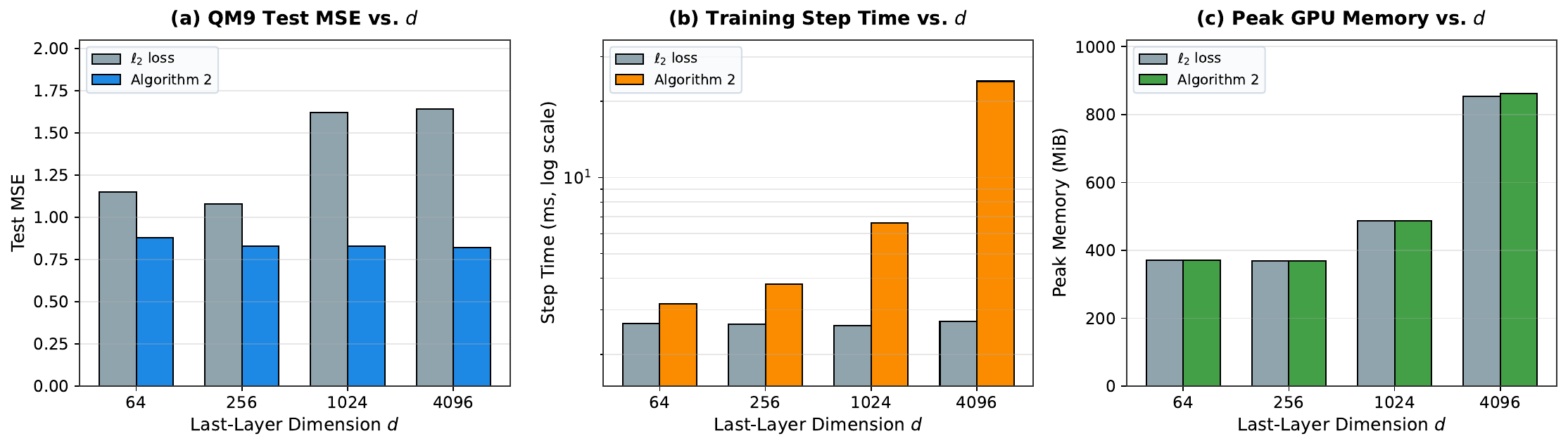}
\caption{\textbf{End-to-end QM9 training ablations ($H=256, L=3, B=128$).} \textbf{(a)}~Test MSE vs.\ $d$. \textbf{(b)}~Training step time vs.\ $d$. \textbf{(c)}~Peak GPU memory vs.\ $d$.}
\label{fig:qm9_end_to_end_ablations}
\end{figure}

\clearpage

\section*{NeurIPS Paper Checklist}

The checklist is designed to encourage best practices for responsible machine learning research, addressing issues of reproducibility, transparency, research ethics, and societal impact. Do not remove the checklist: {\bf The papers not including the checklist will be desk rejected.} The checklist should follow the references and follow the (optional) supplemental material.  The checklist does NOT count towards the page
limit. 

Please read the checklist guidelines carefully for information on how to answer these questions. For each question in the checklist:
\begin{itemize}
    \item You should answer \answerYes{}, \answerNo{}, or \answerNA{}.
    \item \answerNA{} means either that the question is Not Applicable for that particular paper or the relevant information is Not Available.
    \item Please provide a short (1--2 sentence) justification right after your answer (even for \answerNA). 
\end{itemize}

{\bf The checklist answers are an integral part of your paper submission.} They are visible to the reviewers, area chairs, senior area chairs, and ethics reviewers. You will also be asked to include it (after eventual revisions) with the final version of your paper, and its final version will be published with the paper.

The reviewers of your paper will be asked to use the checklist as one of the factors in their evaluation. While \answerYes{} is generally preferable to \answerNo{}, it is perfectly acceptable to answer \answerNo{} provided a proper justification is given (e.g., error bars are not reported because it would be too computationally expensive'' or ``we were unable to find the license for the dataset we used''). In general, answering \answerNo{} or \answerNA{} is not grounds for rejection. While the questions are phrased in a binary way, we acknowledge that the true answer is often more nuanced, so please just use your best judgment and write a justification to elaborate. All supporting evidence can appear either in the main paper or the supplemental material, provided in appendix. If you answer \answerYes{} to a question, in the justification please point to the section(s) where related material for the question can be found.

IMPORTANT, please:
\begin{itemize}
    \item {\bf Delete this instruction block, but keep the section heading ``NeurIPS Paper Checklist"},
    \item  {\bf Keep the checklist subsection headings, questions/answers and guidelines below.}
    \item {\bf Do not modify the questions and only use the provided macros for your answers}.
\end{itemize}


\begin{enumerate}

\item {\bf Claims}
    \item[] Question: Do the main claims made in the abstract and introduction accurately reflect the paper's contributions and scope?
    \item[] Answer: \answerYes{} 
    \item[] Justification: The theoretical claims are verified in \cref{sec:method,sec:stochastic,sec:theory} as well as \cref{sec_app:theory}, while the experimental claims are verified in \cref{sec:experiments} as well as in \cref{sec_app:additional_results}.

\item {\bf Limitations}
    \item[] Question: Does the paper discuss the limitations of the work performed by the authors?
    \item[] Answer: \answerYes{} 
    \item[] Justification: Limitations are discussed in \cref{sec:conclusion}.

\item {\bf Theory assumptions and proofs}
    \item[] Question: For each theoretical result, does the paper provide the full set of assumptions and a complete (and correct) proof?
    \item[] Answer: \answerYes{}
    \item[] Justification: The assumptions are described in the main text (\cref{sec:method,sec:stochastic,sec:theory}), and any remaining proofs are in the appendix (\cref{sec_app:theory}).

    \item {\bf Experimental result reproducibility}
    \item[] Question: Does the paper fully disclose all the information needed to reproduce the main experimental results of the paper to the extent that it affects the main claims and/or conclusions of the paper (regardless of whether the code and data are provided or not)?
    \item[] Answer: \answerYes{} 
    \item[] Justification: Details on the experiments can be found in \cref{sec:experiments} and \cref{sec_app:experimental_details}.

\item {\bf Open access to data and code}
    \item[] Question: Does the paper provide open access to the data and code, with sufficient instructions to faithfully reproduce the main experimental results, as described in supplemental material?
    \item[] Answer: \answerNo{} 
    \item[] Justification: While we do not include the code implementation, we include extensive experimental details in~\cref{sec_app:experimental_details}. Moreover, implementing all the algorithms from the paper is straightforward by following the steps in~\Cref{alg:method} or in~\Cref{alg:method_simple}, or in~\Cref{alg:dfiv}.

\item {\bf Experimental setting/details}
    \item[] Question: Does the paper specify all the training and test details (e.g., data splits, hyperparameters, how they were chosen, type of optimizer) necessary to understand the results?
    \item[] Answer: \answerYes{} 
    \item[] Justification: All experimental details can be found in \cref{sec:experiments} and \cref{sec_app:experimental_details}.

\item {\bf Experiment statistical significance}
    \item[] Question: Does the paper report error bars suitably and correctly defined or other appropriate information about the statistical significance of the experiments?
    \item[] Answer: \answerYes{}
    \item[] Justification: Shaded areas are used as means of error bars in \cref{sec:experiments}.

\item {\bf Experiments compute resources}
    \item[] Question: For each experiment, does the paper provide sufficient information on the computer resources (type of compute workers, memory, time of execution) needed to reproduce the experiments?
    \item[] Answer: \answerYes{} 
    \item[] Justification: We provide details regarding the hardware in~\Cref{sec_app:experimental_details}. Moreover, every experiment has an extensive amount of details about number of iterations (epochs) and the considered hyperparameters.
    
\item {\bf Code of ethics}
    \item[] Question: Does the research conducted in the paper conform, in every respect, with the NeurIPS Code of Ethics \url{https://neurips.cc/public/EthicsGuidelines}?
    \item[] Answer: \answerYes{}
    \item[] Justification: We fulfill all requirements.

\item {\bf Broader impacts}
    \item[] Question: Does the paper discuss both potential positive societal impacts and negative societal impacts of the work performed?
    \item[] Answer: \answerNA{}
    \item[] Justification: Our paper conducts no societally harmful research.
    
\item {\bf Safeguards}
    \item[] Question: Does the paper describe safeguards that have been put in place for responsible release of data or models that have a high risk for misuse (e.g., pre-trained language models, image generators, or scraped datasets)?
    \item[] Answer: \answerNA{}
    \item[] Justification: There is no risk of misuse associated to the models used in the present work.

\item {\bf Licenses for existing assets}
    \item[] Question: Are the creators or original owners of assets (e.g., code, data, models), used in the paper, properly credited and are the license and terms of use explicitly mentioned and properly respected?
    \item[] Answer: \answerYes{}
    \item[] Justification: We credit the creators of each dataset used in \cref{sec:experiments,sec_app:additional_results}.

\item {\bf New assets}
    \item[] Question: Are new assets introduced in the paper well documented and is the documentation provided alongside the assets?
    \item[] Answer: \answerNA{}
    \item[] Justification: We do not release new assets.

\item {\bf Crowdsourcing and research with human subjects}
    \item[] Question: For crowdsourcing experiments and research with human subjects, does the paper include the full text of instructions given to participants and screenshots, if applicable, as well as details about compensation (if any)? 
    \item[] Answer: \answerNA{}
    \item[] Justification: We do not use crowdsourcing.

\item {\bf Institutional review board (IRB) approvals or equivalent for research with human subjects}
    \item[] Question: Does the paper describe potential risks incurred by study participants, whether such risks were disclosed to the subjects, and whether Institutional Review Board (IRB) approvals (or an equivalent approval/review based on the requirements of your country or institution) were obtained?
    \item[] Answer: \answerNA{}
    \item[] Justification: We do not conduct experiments on humans.

\item {\bf Declaration of LLM usage}
    \item[] Question: Does the paper describe the usage of LLMs if it is an important, original, or non-standard component of the core methods in this research? Note that if the LLM is used only for writing, editing, or formatting purposes and does \emph{not} impact the core methodology, scientific rigor, or originality of the research, declaration is not required.
    \item[] Answer: \answerNA{} 
    \item[] Justification: We have not made use of LLMs in any significant way.

\end{enumerate}

\end{document}